\documentclass[sigconf]{aamas}

\usepackage{balance} 

\usepackage[dvipsnames]{xcolor}
\usepackage{svg}
\usepackage{multirow}
\usepackage{subcaption}
\usepackage{amsthm}
\usepackage{algorithm}
\usepackage[noend]{algpseudocode}
\usepackage{bbm}
\usepackage{xspace}

\doi{AGSA2170}

\makeatletter
\gdef\@copyrightpermission{
  \begin{minipage}{0.2\columnwidth}
   \href{https://creativecommons.org/licenses/by/4.0/}{\includegraphics[width=0.90\textwidth]{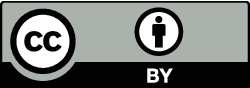}}
  \end{minipage}\hfill
  \begin{minipage}{0.8\columnwidth}
   \href{https://creativecommons.org/licenses/by/4.0/}{This work is licensed under a Creative Commons Attribution International 4.0 License.}
  \end{minipage}
  \vspace{5pt}
}
\makeatother

\setcopyright{ifaamas}
\acmConference[AAMAS '26]{Proc.\@ of the 25th International Conference
on Autonomous Agents and Multiagent Systems (AAMAS 2026)}{May 25 -- 29, 2026}
{Paphos, Cyprus}{C.~Amato, L.~Dennis, V.~Mascardi, J.~Thangarajah (eds.)}
\copyrightyear{2026}
\acmYear{2026}
\acmDOI{}
\acmPrice{}
\acmISBN{}

\title[OSIL: Learning Offline Safe Imitation Policy]{OSIL: Learning Offline Safe Imitation Policies with Safety Inferred from Non-preferred Trajectories}

\author{Returaj Burnwal}
\affiliation{
  \institution{Indian Institute of Technology Madras}
  \city{Chennai}
  \country{India}}
\email{cs21d406@smail.iitm.ac.in}

\author{Nirav Pravinbhai Bhatt}
\affiliation{
  \institution{Indian Institute of Technology Madras}
  \city{Chennai}
  \country{India}}
\email{niravbhatt@dsai.iitm.ac.in}

\author{Balaraman Ravindran}
\affiliation{
  \institution{Indian Institute of Technology Madras}
  \city{Chennai}
  \country{India}}
\email{ravi@dsai.iitm.ac.in}

\begin{abstract}
    This work addresses the problem of offline safe imitation learning (IL), where the goal is to learn safe and reward-maximizing policies from demonstrations that do not have per-timestep safety cost or reward information. In many real-world domains, online learning in the environment can be risky, and specifying accurate safety costs can be difficult. However, it is often feasible to collect trajectories that reflect undesirable or unsafe behavior, implicitly conveying what the agent should avoid. We refer to these as non-preferred trajectories. We propose a novel offline safe IL algorithm, \our, that infers safety from non-preferred demonstrations. We formulate safe policy learning as a Constrained Markov Decision Process (CMDP). Instead of relying on explicit safety cost and reward annotations, \our reformulates the CMDP problem by deriving a lower bound on reward maximizing objective and learning a cost model that estimates the likelihood of non-preferred behavior. Our approach allows agents to learn safe and reward-maximizing behavior entirely from offline demonstrations. We empirically demonstrate that our approach can learn safer policies that satisfy cost constraints without degrading the reward performance, thus outperforming several baselines.
\end{abstract}

\keywords{Imitation Learning, Offline RL, Safe RL}

\newcommand{\BibTeX}{\rm B\kern-.05em{\sc i\kern-.025em b}\kern-.08em\TeX}

\theoremstyle{plain}

\newtheorem{lemma}{Lemma}
\newtheorem{theorem}{Theorem}

\newtheorem{definition}{Definition}

\newtheorem*{theorem*}{Theorem}

\newcommand{\our}{OSIL\xspace}
\newcommand\aeq{\mathrel{\overset{\makebox[0pt]{\mbox{\normalfont\tiny\sffamily (a)}}}{=}}}

\definecolor{lightblue}{rgb}{0.8,0.9,1}
\newcommand{\highlight}[2][lightblue]{\mathchoice
  {\colorbox{#1}{$\displaystyle#2$}}
  {\colorbox{#1}{$\textstyle#2$}}
  {\colorbox{#1}{$\scriptstyle#2$}}
  {\colorbox{#1}{$\scriptscriptstyle#2$}}}

\begin{document}


\pagestyle{fancy}
\fancyhead{}


\maketitle 


\section{Introduction}
\label{sec:introduction}
\begin{figure*}[t!]
    \centering
    \includegraphics[width=0.95\textwidth]{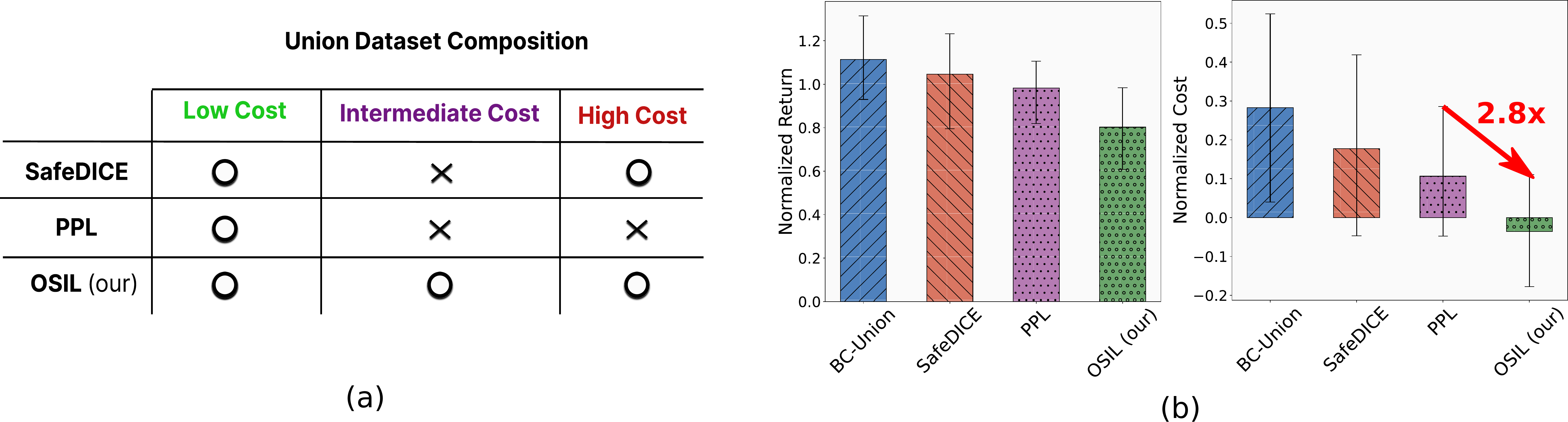}
    \caption{(a) PPL works well when the union dataset consists of {\color{ForestGreen} \textbf{low-cost}} trajectories and SafeDICE requires that the union dataset contains {\color{ForestGreen} \textbf{low-cost}} or {\color{Red} \textbf{high-cost}} trajectories. This assumption is often unrealistic in practice since trajectories collected from real-world typically span a spectrum of safety costs. \our do not make any assumptions about the safety cost of a trajectory in the union dataset, as it may contain trajectories with varying levels of safety cost. (b) Under this setting, we observe that \our learns a high-return, safer (i.e., low-cost) policy, outperforming the best baseline by nearly 2.8x. We report the mean performance of the algorithm after 1 million training steps, aggregated across both velocity-constrained and navigation tasks. Mean and 95\% CIs over 5 seeds.}
    \label{fig:aggregated_performance}
\end{figure*}

Reinforcement Learning (RL) provides a framework for enabling agents to autonomously learn intelligent behavior through interaction with their environment. While RL has recently demonstrated remarkable success in various simulated tasks \cite{humanDeepRL,masteringGo,dota2}, its real-world deployment remains limited \cite{robot_challenges,social_challenges,safedice_neurips_2023} due to the following reasons. First, RL typically requires many online interactions with the environment to learn a good policy. This reliance on data collection poses a significant barrier in domains where such interactions are costly or risky, such as robotics or autonomous driving. Second, designing a suitable reward function that accurately reflects the desired behavior is often difficult in practice. Since the agent learns by maximizing the reward, poorly defined reward functions can lead to unintended and potentially dangerous behavior \cite{ill_defined_reward_aaai_2014, ill_defined_reward_iclr_2020}. Third, many real-world applications necessitate adherence to safety constraints alongside reward maximization. Safety in RL is usually modeled as a Constrained Markov Decision Process (CMDP) \cite{cmdp_1999}. In CMDP, the agent aims to maximize rewards over time while satisfying safety cost constraints. However, specifying appropriate per-timestep constraint cost information can be equally difficult. For instance, defining constraint cost information for assessing toxicity in conversational agents or complex domains like surgical robotics can be challenging.

Consider a real-world scenario of autonomous driving where we have access to a large collection of human-driven trajectories that are high-return, but exhibit varying degrees of safety cost. Standard imitation learning (IL) algorithms \cite{il_1_neurips_1996, il_2_elsevier_2009} can learn such high-return behavior. However, in most real-world applications, it is often necessary to ensure that learned behaviors also satisfy safety constraints.
In this paper, we address the problem of offline safe IL, when given a large set of high-return trajectories with varying degrees of safety cost, \textit{union trajectory dataset}, and a small set of \textit{non-preferred trajectory dataset} that are high-return and high-cost. We aim to learn a policy that infers safety from these non-preferred demonstrations. The non-preferred trajectories that violate safety constraints are often naturally collected. For example, instances of illegal driving, like running red lights, can be used as non-preferred trajectories. This setting has mostly been unexplored and cannot be effectively addressed using the existing methods.  The recent work on offline safe IL,  SafeDICE \cite{safedice_neurips_2023}, learns a safe policy under this setting. However, the SafeDICE makes a limiting assumption: that the \textit{union dataset}, high-return trajectories, contains either low-cost or high-cost behavior. In practice, high-return trajectories may not fit neatly into this binary categorization. Another relevant approach is based on preference-based policy learning (PPL) \cite{trex_icml_2019,b_pref_neurips_2021,offline_pref_tmlr_2023,pebble_icml_2021}, which learns a policy based on preferences over pairs of demonstrations. PPL can be used to train a policy that favors \textit{union trajectories} over non-preferred ones. However, this approach assumes that \textit{union trajectories} are safe, which is not guaranteed. A more realistic assumption would be to allow these high-return \textit{union trajectories} to have varying degrees of safety cost.

We propose a novel offline safe IL algorithm, \our, that infers safety from these non-preferred demonstrations. Our key contributions are as follows: 1) We formulate the safe policy learning as a CMDP. Since we do not have access to reward and cost information, we reformulate the CMDP problem by deriving a lower bound of the reward maximizing objective and approximating the safety constraints by learning a parameterized cost model that estimates the likelihood of a state-action pair being non-preferred. 2) We then solve the reformulated CMDP problem using a Lagrangian relaxation with an adaptive penalty coefficient to obtain a policy that effectively balances safety and performance. 3) We empirically demonstrate that \our learns safer policies (Figure \ref{fig:aggregated_performance}) than existing state-of-the-art offline safe IL methods, while maintaining competitive task success.


\section{Related Works}
Offline Imitation Learning \cite{bc,valuedice,demodice_iclr_2022} primarily focuses on replicating actions of the expert demonstrations without explicitly considering safety. It implicitly assumes that expert demonstrations are safe, but if they contain non-preferred trajectories, simply imitating them can lead to learning policies that are non-preferred. We address the problem of learning safe imitation policy by seeking limited access to non-preferred trajectories and an abundance of high-return union trajectories that contain trajectories with varying costs. This specific scenario has received limited attention in the literature.

\noindent \textbf{Learning Imitation Policy from Non-Preferred Trajectories}: SafeDICE \cite{safedice_neurips_2023} directly estimating the stationary distribution of the low-cost policy using non-preferred and union trajectory dataset. However, it assumes that the union dataset contain either low-cost or high-cost trajectories, a simplification that often does not hold in practice. The estimated stationary distribution is then used to learn a safe policy that avoids high-cost behaviors. 

\noindent \textbf{Learning Imitation Policy from Suboptimal Trajectories}: Preference based policy learning (PPL) methods, such as T-REX \cite{trex_icml_2019}, B-Pref \cite{b_pref_neurips_2021}, PEBBLE \cite{pebble_icml_2021}, OPRL \cite{offline_pref_tmlr_2023}, aim to learn a reward function from ranked trajectories. These methods prioritize learning reward functions that assign higher total reward to higher ranked trajectories in the dataset. T-REX, B-Pref, PEBBLE then uses online-RL \cite{ppo,sac} algorithm, and OPRL uses offline-RL \cite{cql_neurips_2020,neorl_neurips_2022} algorithm to learn the policy. Since ranked trajectories are difficult to obtain, D-REX \cite{drex_corl_2020} and SSRR \cite{ssrr_corl_2020} both learn reward function by generating trajectories of varying optimality by injecting noise to the policy learned from suboptimal demonstrations. However, the trajectory generation requires online interaction with the environment, which is not feasible in an offline setting. Another method, DWBC \cite{dwbc_icml_2022} uses positive unlabeled learning \cite{pu_learning_1_2008,pu_learning_2_2021} to train a discriminator network. This discriminator network is subsequently used as a weight in the BC loss function.


\section{Preliminaries}
\begin{figure*}[!t]
    \centering
    \includegraphics[width=0.8\textwidth]{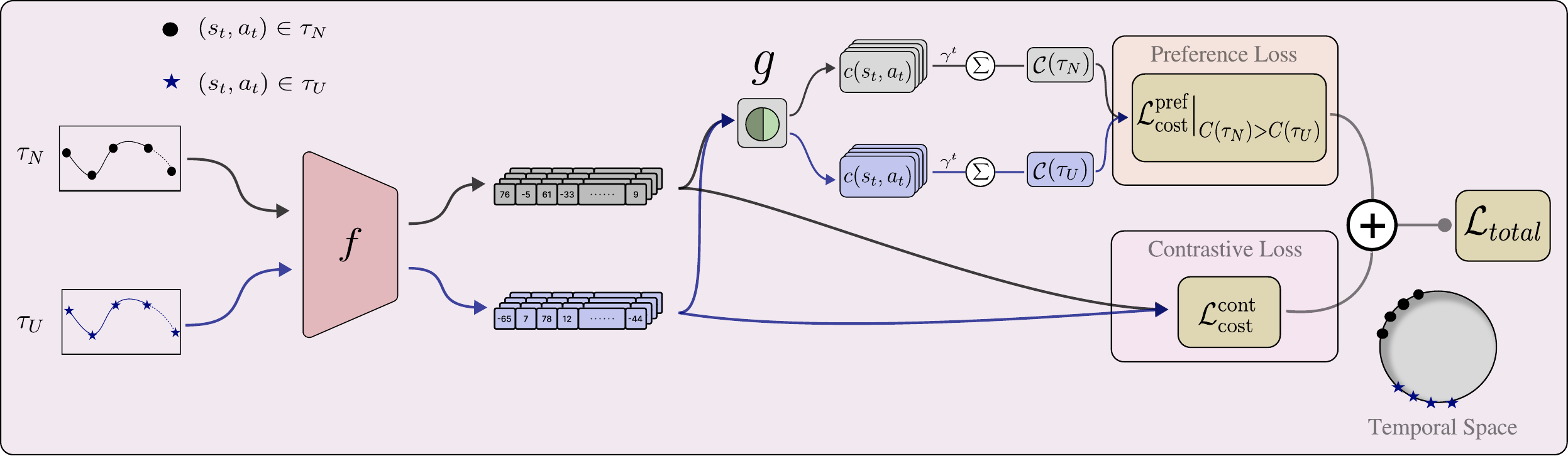}
    \caption{Overview of the cost learning model. $f$ and $g$ are a learnable encoder and a linear model, respectively. The cost model is trained by minimizing the two loss function: $\mathcal{L}_{\text{cost}}^\text{const}$ encourages temporally adjacent state-action pairs within a trajectory to remain close in the learned representation space, and $\mathcal{L}_\text{cost}^\text{pref}$ ensures that the discounted cost of the trajectory $\tau_N$ is greater than trajectory $\tau_U$.}
    \label{fig:method_cost_learning}
\end{figure*}

\subsection{Constrained Markov Decision Process}
Consider a finite Constrained Markov Decision Process (CMDP) \cite{cmdp_1999} represented by the tuple $\mathcal{M} = (\mathcal{S},\mathcal{A},\mathcal{T}, r, \rho_0, \gamma, \mathcal{C})$, where $\mathcal{S}$ and $\mathcal{A}$ represent state and action spaces, respectively. $\mathcal{T}: \mathcal{S}\times\mathcal{A} \rightarrow \Delta(\mathcal{S})$ is the transition probability, $r: \mathcal{S}\times\mathcal{A} \rightarrow \mathbb{R}$ denotes the immediate reward function, $\rho_0 \in \Delta(\mathcal{S})$ is the initial state distribution, and $\gamma \in (0,1)$ is the discount factor. $\mathcal{C} = \{(c_i,b_i)\}_{i=1}^m$ is a constraint set, where $c_i: \mathcal{S}\times\mathcal{A} \rightarrow \mathbb{R}_{\geq 0}$ is the $i$-th cost function and $b_i\in\mathbb{R}_{\geq 0}$ is the corresponding threshold. A stochastic policy $\pi: \mathcal{S} \rightarrow \Delta(\mathcal{A})$ corresponds to a map from state to a probability distribution over actions. The reward value function and action-value function  under policy $\pi$ is defined as $V_r^\pi(s) = \mathbb{E}_{\tau \sim \pi} [\sum_{t=0}^{T-1} \gamma^t r(s_t, a_t) | s_0=s]$ and $Q_r^\pi(s) = \mathbb{E}_{\tau \sim \pi} [\sum_{t=0}^{T-1} \gamma^t r(s_t, a_t) | s_0=s, a_0=a]$, where $\tau = (s_0, a_0, \dots, s_{T-1}, a_{T -1}, s_T)$ is a $T$ length trajectory sampled under policy $\pi$.
Similarly, the $i$-th cost action-value function is defined as $Q_{c_i}^\pi(s,a) = \mathbb{E}_{\tau \sim \pi}[\sum_{t=0}^{T-1} \gamma^t c_i(s_t, a_t) | s_0=s,a_0=a]$.
Therefore, the reinforcement learning problem in CMDP is 
\begin{align}
    \label{eq:CMDP}
    \pi^\star =& arg\max_\pi\; J_r(\pi):=\mathbb{E}_{\tau \sim \pi}\left[ \sum_{t=0}^{T-1} \gamma^t r(s_t, a_t)\right]  \nonumber \\
    & \text{s.t.}\;\; \mathbb{E}_{s\sim\rho_0}\left[ Q^\pi_{c_i}(s, \pi(s)) \right] \leq b_i, \;\ \forall i
\end{align}
where $J_r(\pi)$ denotes the  the performance of the policy $\pi$ and $\mathbb{E}_{s\sim\rho_0}\left[ Q^\pi_{c_i}(s, \pi(s)) \right] \leq b_i$ is the $i$-th constraint. Imitation learning does not rely on environmental rewards but rather on demonstrations (trajectories). This work focuses on an offline setting, where the policy is learned solely from pre-collected trajectories that do not contain reward and cost information explicitly, i.e., $\tau = (s_0, a_0, \dots, s_{T-1}, a_{T -1}, s_T)$. In this work, non-preferred and union trajectories are defined as follows:
\begin{definition}[Non-preferred Trajectory Dataset]
    A non-preferred trajectory dataset contains trajectories that achieve high return but violate multiple constraints by a large margin. In other words, it is a high-return, high-cost trajectory dataset.
\end{definition}
\begin{definition}[Union Trajectory Dataset]
    A union trajectory dataset contains trajectories that achieve high return with varying degrees of constraint violation. In other words, it is a high-return, varying-cost trajectory dataset.
\end{definition}

We assume we have access to a limited number of non-preferred trajectories, $\mathcal{D}_N$, and a large number of union trajectories, $\mathcal{D}_U$. Additionally, we represent the policy that generated the union trajectory dataset as $\pi_U$. This means the actions in the trajectory, $\tau \in \mathcal{D}_U$, are sampled from $\pi_U$.  

\subsection{Contrastive Learning}
The procedure starts with a set of $N$ randomly sampled labeled examples, $\{(x_i, y_i)\}_{i=1}^N$. For each sample, we generate two random augmentations, resulting in a batch of $2N$ augmented pairs: $\mathcal{B} = \{(\tilde{x}_i, \tilde{y}_i)\}_{i=1}^{2N}$, where for each $k \in \{1, 2, \dots, N\}$, the augmented inputs $\tilde{x}_{2k-1}$ and $\tilde{x}_{2k}$ are derived from $x_k$, and both share the same label: $\tilde{y}_{2k-1} = \tilde{y}_{2k} = y_k$. This batch of $2N$ pairs is then used to train an encoder $f(\cdot)$ that maps input $x$ to a $d$-dimensional unit-norm representation vectors $z = f(x) \in \mathbb{R}^d$.

\textbf{Supervised Contrastive Training} \cite{supcon_neurips_2020}: We train the encoder model $f$ using a batch $\mathcal{B}$ of $2N$ augmented labeled sample pairs. The training objective of the supervised contrastive loss is defined as:
\begin{align} 
    \label{eq:contrastive_supervised}
    \mathcal{L}^{sup} = \sum_{i \in \mathcal{B}} \dfrac{-1}{|P(i)|} \sum_{p \in P(i)} \log \dfrac{\exp(z_i^\top z_p / \eta)}{\sum_{b \in \mathcal{B} \setminus {i}} \exp(z_i^\top z_b / \eta)} 
\end{align}
Here, $z_i = f(\tilde{x}_i)$ is the encoded representation of the augmented input $\tilde{x}_i$, $P(i) = \{p \in \mathcal{B} \setminus {i} \mid \tilde{y}_p = \tilde{y}_i\}$ is the set of all positive samples in the batch that share the same label as $\tilde{x}_i$, $|P(i)|$ denotes its cardinality, and $\eta \in \mathbb{R}^+$ is a temperature scaling parameter. 
This contrastive loss enables the model to learn label-aware representations by pulling together samples from the same class in the embedding space, thereby promoting class-consistent representations.


\section{Methodology}
This section introduces \our method which formulates safe policy learning as a CMDP problem, defined in Equation \ref{eq:CMDP}. In general, both the reward and safety cost functions are difficult to define, and in our work, we do not assume access to them. Instead, we assume access to a limited number of non-preferred trajectories $\mathcal{D}_N$ and union trajectories $\mathcal{D}_U$. In Section \ref{sec:cost_action_value_function}, we first approximate the safety cost functions using a parameterized cost model that estimates the likelihood of a state-action pair being non-preferred and use it to learn the cost action-value function. Finally, in Section \ref{sec:policy_learning} we provide a lower bound on policy’s performance, $J_r(\pi)$, that does not depend on reward labels and solves the resulting optimization problem to obtain a policy that effectively balances safety and performance.

\subsection{Learning Cost Action-Value Function}
\label{sec:cost_action_value_function}
To learn a cost action-value function, we first need to learn a cost model that effectively captures the safety cost information. We define our parameterized cost model as a composition $\tilde{c}:= g \circ f$, where $f: \mathcal{S}\times \mathcal{A} \rightarrow \mathbb{R}^d$ is an encoder model that maps state-action pairs to a $d$-dimensional unit-norm latent representation, and $g: \mathbb{R}^d \rightarrow [0,1]$ is a linear model that maps the encoded representation to a scalar value (See Figure~\ref{fig:method_cost_learning}).

\textbf{Loss function for cost model:} The encoder model $f$ is trained such that the state-action pairs that appear in the same trajectory have similar representations, thereby encouraging the model to capture temporal dependencies. We use partial trajectory of length $H << T$, rather than full trajectory of length $T$. This has two benefits: (1) short partial trajectories better capture local temporal correlations, as nearby state-action pairs tend to be more similar; and (2) shorter sequences reduce computational overhead during training. Given a batch of trajectories from the union dataset $\mathcal{B}_U \sim \mathcal{D}_U$, and the non-preferred dataset $\mathcal{B}_N \sim \mathcal{D}_N$, we train the encoder using a contrastive loss:
\begin{align}
    \label{eq:cost_contrastive}
    \mathcal{L}^{\text{cont}}_{\text{cost}} =& \sum_{\tau \in \mathcal{B}_U \cup\mathcal{B}_N} \sum_{t \in \tau} \dfrac{-1}{|\tau| -1} \sum_{t' \in \tau \setminus t} \log \dfrac{\exp(z^\top_t z_{t'}/\eta)}{\sum_{\tau \in \mathcal{B}_U \cup \mathcal{B}_N} \sum_{k \in \tau} \exp(z^\top_t z_{k}/\eta)}
\end{align}
Here, $t\in \tau$ denotes a state-action pair $(s_t,a_t)$ in trajectory $\tau$, $z_t = f(s_t, a_t)$ is its encoded representation, and $|\cdot|$ represents the cardinality. Minimizing $\mathcal{L}^{\text{cont}}_{\text{cost}}$ ensures that the state-action pairs in the same trajectory are temporally close. 

The output of the encoder model is then passed to the linear model $g$ that estimates the likelihood of the state-action pair being non-preferred. Given a trajectory $\tau$ we define the total discounted cost of a trajectory as $C(\tau) = \sum_{t=0}^{T-1} \gamma^t \tilde{c}(s_t, a_t)$ where $\tilde{c}:= g \circ f$ is our cost model. Since the non-preferred trajectories typically incur higher costs than those from the union dataset, we introduce a preference-based training objective to encourage the cost-model to assign higher costs to non-preferred trajectory $\tau_N$ than to union trajectory $\tau_U$, i.e., $C(\tau_N) > C(\tau_U)$. To formalize this preference, we define the loss function based on the Bradley–Terry model \cite{bradley_terry_1952} as:
\begin{align} 
    p_{\text{non}} &= p(\tau_N \succ \tau_U) = \dfrac{\exp(C(\tau_N))}{\exp(C(\tau_N)) + \exp(C(\tau_U))} \\\nonumber\\
    \label{eq:cost_hard} 
    \mathcal{L}^{\text{pref}}_{\text{cost}} &= \mathbb{E}_{\tau_N, \tau_U} \left[ -\log\; p_{\text{non}} \right] = \mathbb{E}_{\tau_N, \tau_U} \left[\text{BCE}(p_{\text{non}}; \mathbbm{1})\right]     
\end{align}
Here, $p_{\text{non}}$ represents the prediction probability that $\tau_N$ has higher cost over $\tau_U$, and $\text{BCE}(\cdot,\cdot)$ denotes binary cross-entropy loss function with a hard target label as $\mathbbm{1}$, indicating a strong supervision signal that $\tau_N$ should have higher cost than $\tau_U$. While this hard labeling assumption may not always strictly hold, since some of the trajectories in the union dataset are high-cost, we observe in our experiments that using this hard labeling does not degrade the performance. This is likely because, on average, trajectories sampled from the union dataset tend to have lower cost than those from the non-preferred dataset, making the pairwise ranking signal reliable in practice. 

Finally, we jointly train our cost model by minimizing both the contrastive and the preference-based loss function:
\begin{align}
    \label{eq:cost_loss}
    \mathcal{L}_{\text{cost}} = \mathcal{L}^{\text{pref}}_{\text{cost}} + \mathcal{L}^{\text{cont}}_{\text{cost}}
\end{align}

\textbf{Loss function for cost action-value:} We train our cost value function $Q^\pi_c: \mathcal{S}\times\mathcal{A} \rightarrow \mathbb{R}$ for a given policy $\pi$, using a learned cost model over the union dataset. The cost action-value function $Q^\pi_{\tilde{c}}(s, a)$ represents the expected long-term cost incurred by taking action $a$ in state $s$ and subsequently following policy $\pi$. To estimate $Q^\pi_c$, we iteratively minimize the following squared temporal-difference (TD) loss:
\begin{align}
    \label{eq:cost_action_value_loss}
    \mathcal{L}_{\text{value}} = \mathbb{E}_{(s,a,s')\sim \mathcal{D}_U} \left[ (y_{\text{target}}(s,a,s') -  Q^\pi_{\tilde{c}}(s,a))^2 \right]
\end{align}
where $y_{\text{target}}(s, a,s') = {\tilde{c}}(s, a) + \gamma Q^{\pi-}_{\tilde{c}}(s', \pi(s'))$ is the TD target,  $\tilde{c}(s, a)$ is the immediate cost estimated from the learned cost model, and $Q^{\pi-}_{\tilde{c}}$ is the target Q network which is slowly updated using polyak averaging scheme $Q^{\pi-}_{\tilde{c}} \leftarrow (1-\zeta) Q^{\pi-}_{\tilde{c}} + \zeta Q^{\pi}_{\tilde{c}}$ with a fixed smoothing coefficient $\zeta \in (0,1)$.

\begin{figure*}[th!]
    \centering
    \includegraphics[width=0.8\linewidth]{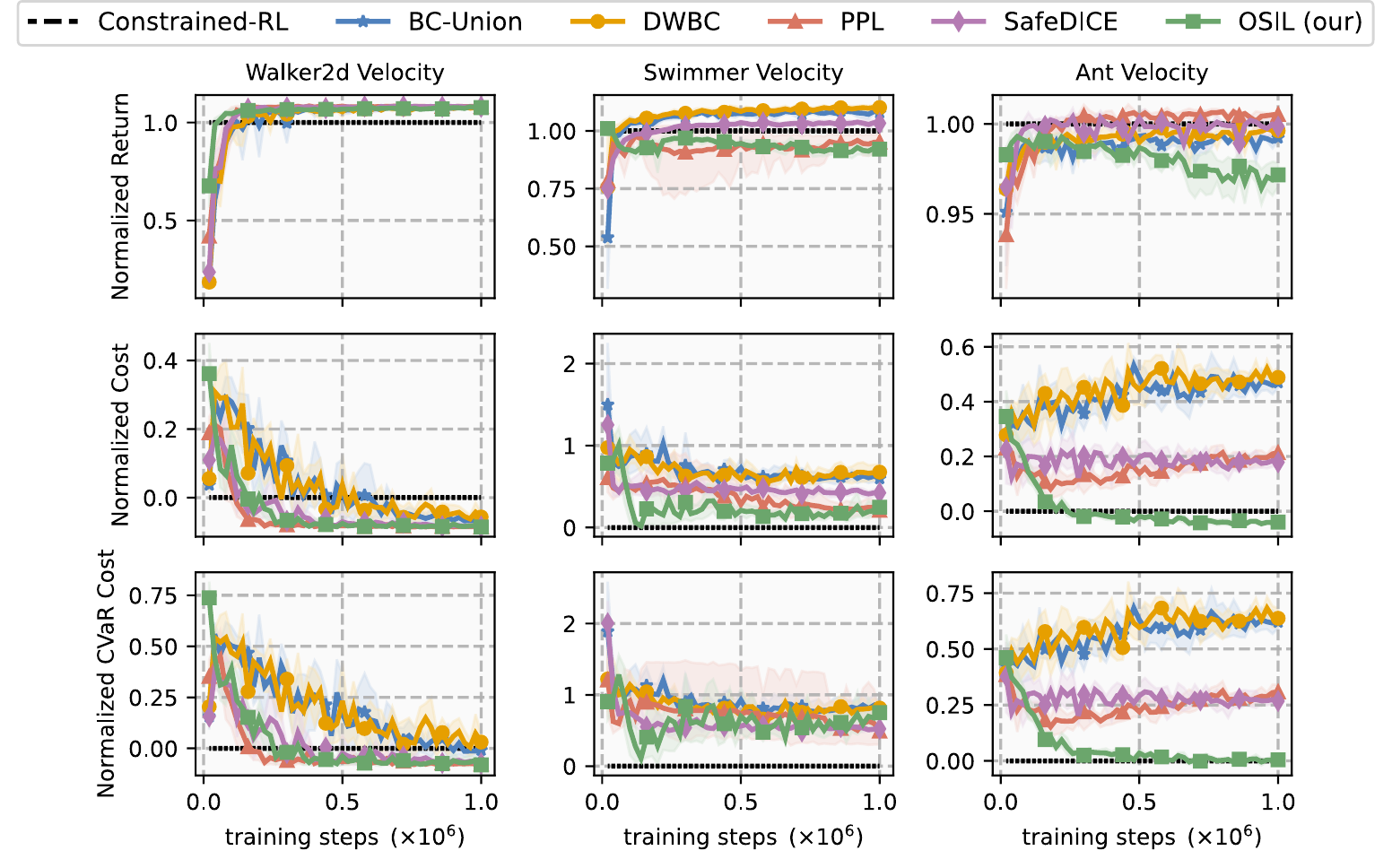}
    \caption{\textbf{Performance Comparison.} Experimental results on Walker2d-Velocity, Swimmer-Velocity, Ant-Velocity task. The shaded area represents the standard error. In velocity-constrained tasks, our method is able to recover safer policies without compromising reward performance.}
    \label{fig:velocity_task_performance}
\end{figure*}

\subsection{Learning Safe Policy}\label{sec:policy_learning}
For a given cost action-value function we can rewrite the reinforcement learning problem for CMDP as:
\begin{align}
    \label{eq:policy_learn_reward_performance}
    arg\max_{\pi} &\;J_r(\pi) := \mathbb{E}_{\tau \sim \pi}\left[ \sum_{t=0}^{T-1} \gamma^t r(s_t, a_t) \right] \nonumber \\
    &\text{s.t.}\;\; \mathbb{E}_{s\sim\rho_0}\left[Q^\pi_{\tilde{c}}(s, \pi(s)) \right] \leq b
\end{align}
Solving this constrained optimization is challenging, as we do not have access to the reward information. However, we know that the union dataset $\mathcal{D}_U$ contains high-return trajectories. We derive a lower bound on the performance of the policy $\pi$ in terms of the performance of the policy $\pi_U$ that generated $\mathcal{D}_U$. 
\begin{theorem}
    \label{theorem:bc_bound}
    Let $\pi_U$ represents the policy that generated union dataset. Define $\epsilon = \max_{s,a} |Q^\pi_r(s,a) - V^\pi_r(s)|$, and $D^{\text{max}}_{\text{KL}}(\pi_U,\pi) = \max_s D_{\text{KL}}(\pi_U(.|s) || \pi(.|s))$. Then the following bound holds:
    \begin{align}
        \label{eq:bc_bound}
        J_r(\pi) \geq J_r(\pi_U) - \dfrac{2\epsilon}{1-\gamma}\sqrt{D_{\text{KL}}^{\text{max}}(\pi_U, \pi)} 
    \end{align}
\end{theorem}
The proof for the theorem is in Appendix \ref{appendix:performance_bound_proof}. Since $\pi_U$ is assumed to produce high-return trajectories, $J_r(\pi_U)$ serves as a proxy for the maximum achievable performance. 

The objective in Equation \ref{eq:policy_learn_reward_performance} can be optimized by maximizing the lower bound on $J_r(\pi)$. This leads to the following reformulation of the objective equation \ref{eq:policy_learn_reward_performance} as:
\begin{align}
    arg\max_\pi \;J_r(\pi) \;\succcurlyeq\;& arg\max_\pi  \; J_r(\pi_U) - C \sqrt{D_{\text{KL}}^{\text{max}}(\pi_U, \pi)} \nonumber \\
    \aeq\; &J_r(\pi_U) -  arg\max_\pi \; C \sqrt{D_{\text{KL}}^{\text{max}}(\pi_U, \pi)} \nonumber \\
    \label{eq:bc_bound_max_kl}
    =\;& arg\min_\pi \; D_{\text{KL}}^{\text{max}}(\pi_U, \pi)
\end{align}
Here, $C = \dfrac{2\epsilon}{1 - \gamma}$, and step (a) follows from the fact that $J_r(\pi_U)$ is the maximum achievable performance. The equation \ref{eq:bc_bound_max_kl} imposes constraint that the KL divergence is bounded at every point in the state space. While this formulation is theoretically grounded, it is often impractical to solve due to the large number of constraints. To make the problem tractable, we approximate the constraint using the average KL divergence.
\begin{align}
    \label{eq:avg_kl}
    \mathbb{E}_{s\sim \mathcal{D}_U}\left[ D_{\text{KL}}(\pi_U || \pi) \right] &=\;\mathbb{E}_{s\sim \mathcal{D}_U}\left[ \mathbb{E}_{a\sim \pi_U(.| s)} \left[\log \dfrac{\pi_U(a|s)}{\pi(a|s)} \right] \right] \nonumber \\
    &=\; \mathbb{E}_{(s,a)\sim \mathcal{D}_U} \left[\log \dfrac{\pi_U(a|s)}{\pi(a|s)} \right]
\end{align}
In Appendix \ref{appendix:avg_kl}, we discuss the implication of this approximation. The CMDP problem, Equation \ref{eq:CMDP}, is then reformulated as:
\begin{align}
    \label{eq:policy_learn_constrained}
    arg\min_\pi &-\mathbb{E}_{(s,a)\sim \mathcal{D}_U} \left[\log \pi(a|s) \right] \nonumber\\
    &\text{s.t.}\;\;  \mathbb{E}_{s\sim\rho_0,a\sim\pi}\left[ Q^\pi_{\tilde{c}}(s, a) \right] \leq b
\end{align} 
We relax the constrained optimization problem Equation \ref{eq:policy_learn_constrained} using the Lagrangian formulation. 
We introduce a penalty coefficient $\alpha$ and solve the following unconstrained optimization problem:
\begin{align}
    \label{eq:policy_learn_relaxed}
    \mathcal{L}_{\text{policy}} &=\; -\mathbb{E}_{(s,a)\sim \mathcal{D}_U} \left[\log \pi(a|s) \right] + \; \mathbb{E}_{s\sim\rho_0}\left[\highlight{\alpha} Q^\pi_{\tilde{c}}(s, \pi(s)) \right] \nonumber\\
    &=\; \underbrace{-\mathbb{E}_{(s,a)\sim \mathcal{D}_U} \left[\log \pi(a|s) \right]}_{\text{Behavior Cloning}} + \; \underbrace{\mathbb{E}_{(s_0,a_0)\sim\mathcal{D}_U}\left[\highlight{\alpha} Q^\pi_{\tilde{c}}(s_0, \pi(s_0)) \right]}_{\text{Cost Critic}} 
\end{align}
In Equation \ref{eq:policy_learn_relaxed}, the loss function consists of two components: a behavior cloning term that encourages the policy $\pi$ to imitate actions from the high-return dataset $\mathcal{D}_U$, and a cost-critic term that penalizes risky behavior by minimizing the cost-action value $Q^\pi_{\tilde{c}}$. This formulation resembles TD3+BC \cite{td3_plus_bc_neurips_2021}, an offline RL algorithm that also combines behavior cloning with value function for policy learning. However, a key distinction lies in the critic: TD3+BC uses a reward action-value function to guide the policy toward high-return actions, while our approach employs a cost-action value function to promote safe actions.

The hyperparameter $\alpha$, in equation \ref{eq:policy_learn_relaxed}, balances the trade-off between safety (minimizing cost-action value) and performance (maximizing reward through imitation). We choose an adaptive scalar $\alpha$ that balances this trade-off:
\begin{align}
    \label{eq:adaptive_alpha}
    \alpha = \dfrac{\bar{\alpha}}{\dfrac{1}{N}\sum_{(s_0,a_0)\sim \mathcal{D}_U} \exp\left(Q^\pi_{\tilde{c}}(s_0,a_0) - Q^\pi_{\tilde{c}}(s_0,\pi(s))\right)}
\end{align}
This adaptive formulation ensures that $\alpha$ increases when the learned policy selects actions that are, on average, more costly than those in the dataset, i.e., $Q^\pi_{\tilde{c}}(s_0,a_0) < Q^\pi_{\tilde{c}}(s_0,\pi(s))$. In such cases, the policy learns safe behavior by placing a higher weight on the cost-critic term. Conversely, when the policy consistently chooses safer actions, i.e., $Q^\pi_{\tilde{c}}(s_0,a_0) \geq Q^\pi_{\tilde{c}}(s_0,\pi(s))$, the value of $\alpha$ decreases, reducing the regularization and allowing the policy to focus more on high-return behavior.


\section{Experiments}
\begin{figure*}[th!]
    \centering
    \includegraphics[width=0.8\linewidth]{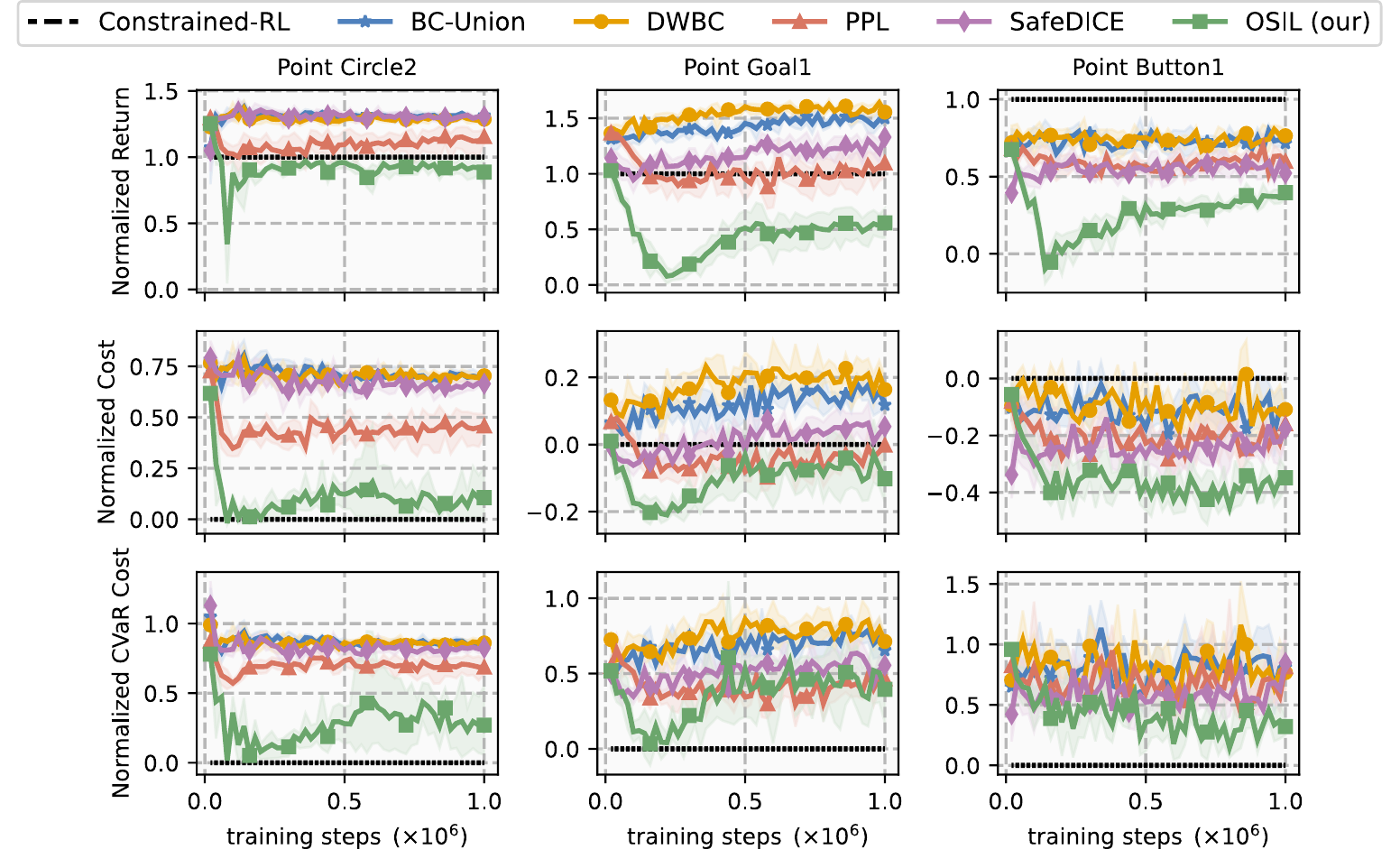}
    \caption{\textbf{Performance Comparison.} Experimental results on Point-Circle2, Point-Goal1, Point-Button1 tasks. The shaded area represents the standard error. Similar to the results in Figure \ref{fig:velocity_task_performance}, our method is able to recover safer policy compared to other baselines.}
    \label{fig:point_task_performance}
\end{figure*}

We evaluate our algorithm against the state-of-the-art offline safe IL method using the Datasets for Offline Safe RL (DSRL) across a suite of benchmark tasks \cite{offline_safe_rl_dataset_jmlr_2024}. Specifically, we run experiments on: (i) MuJoCo-based velocity-constrained tasks (Walker-Velocity, Swimmer-Velocity, and Ant-Velocity), where agents are required to move as fast as possible while respecting velocity limits; and (ii) navigation tasks (Point-Circle2, Point-Goal1, and Point-Button1), where agents aim to maximize performance while avoiding collisions and hazardous areas (See Appendix \ref{appendix:enironment_details} for further details). Together, we conduct experiments on six tasks, providing a comprehensive scenario to assess the safety and effectiveness of offline safe IL algorithms. 

We select a limited number of non-preferred (i.e., high-return-high-cost behavior) trajectories, $\mathcal{D}_N$, and a large number of high-return varying cost union trajectories, $\mathcal{D}_U$ from the DSRL dataset. 
We removed all the reward and cost annotations from the non-preferred ($\mathcal{D}_N$) and the union ($\mathcal{D}_U$) trajectory dataset. Our goal is to recover low-cost behavior while preserving high-reward performance. Additional details about the $\mathcal{D}_N$ and $\mathcal{D}_U$ datasets for each environment are provided in Appendix \ref{appendix:dataset_details}.

\noindent\textbf{Baselines:} We compare our algorithm against three baselines. \textbf{(1) BC-Union}, a BC policy on the union dataset $\mathcal{D}_U$. This serves as a baseline to evaluate the performance of BC when the dataset contains high-return and varying cost trajectories. \textbf{(2) DWBC}, a variant of the discriminator-weighted behavior cloning algorithm \cite{dwbc_icml_2022} that uses \textit{negative-unlabeled} learning to train the discriminator model. The trained discriminator is then used as weight in the weighted BC loss function. \textbf{(3) PPL} \cite{trex_icml_2019,b_pref_neurips_2021,offline_pref_tmlr_2023,pebble_icml_2021}, a preference-based method that learns a reward function to prefer union trajectories over non-preferred ones. The learned reward function is then used as a weight in the weighted BC loss function. \textbf{(4) SafeDICE} \cite{safedice_neurips_2023}, method directly estimates the stationary distribution corrections for the low-cost behavior and then trains a weighted BC policy. Implementation details for all the baseline algorithms are provided in Appendix \ref{appendix:implementation_details}.

Additionally, we train an offline constrained RL policy, Constrained Decision Transformer \cite{cdt_icml_2023}, using the union dataset augmented with ground-truth reward and cost annotations. In our work, we use Constrained-RL method, to estimate the safe performance under complete information. We compare \our, which operates under stricter information constraints, to assess how close its performance is to the Constrained-RL approach. Constrained-RL method is not a baseline. No other methods in our experiments have access to reward and cost information.

We evaluate and compare ours and other baselines algorithms in terms of task performance (measured by expected episodic return) and safety (measured by expected episodic cost). We report the results with the following metrics: \textbf{(1) Normalized Return}: scales the mean episodic return of a given policy, such that ``$0$'' represents the episodic return from a random policy, while ``$1$'' represents the episodic return achieved by the Constrained-RL policy. \textbf{(2) Normalized Cost}: scales the given policy's mean episodic cost, such that ``$0$'' represents the cost achieved by the Constrained-RL policy, and ``$1$'' represents the maximum episodic cost. \textbf{(3) Normalized Conditional Value at Risk performance (CVaR) 20\% Cost}: similar to \textit{Normalized Cost} it scales the policy's mean episodic cost of the worst 20\% runs.

Our primary focus in the evaluation is on safety, measured by \textit{Normalized Cost} and \textit{Normalized CVaR Cost}. We use \textit{Normalized Return} to examine whether the agents can achieve high-return behaviors. We consider a policy to have successfully recovered the low-cost behavior if its performance closely matches the Constrained-RL policy. This is evidenced by the \textit{Normalized Cost} and \textit{Normalized CVaR Cost} less than or close to ``$0$'', and the \textit{Normalized Return} being greater than or close to ``$1$''. All plots are generated by averaging the performance of 50 trajectories generated from the learned policy. To assess statistical significance, we generate 1000 bootstrap samples from the data, using results from 5 different random seeds, and we plot the resulting 95\% confidence intervals.

We answer the following questions through our experiments:
\begin{enumerate}
    \item \textit{Performance Comparison}: How does our algorithm perform relative to other baselines in learning safe policies that satisfy cost constraints while preserving high reward?
    \item \textit{Impact of Non-Preferred Trajectory Dataset Size}: How does performance vary with the size of the non-preferred trajectory dataset $\mathcal{D}_N$?
    \item \textit{Impact of Contrastive Loss}: What is the effect of removing contrastive loss from the cost model training?
    \item \textit{Sensitivity to Trajectory Length}: How does the choice of sub trajectory length affect the algorithm’s performance?
\end{enumerate}

\begin{figure*}[t!]
    \centering
    \includegraphics[width=0.83\linewidth]{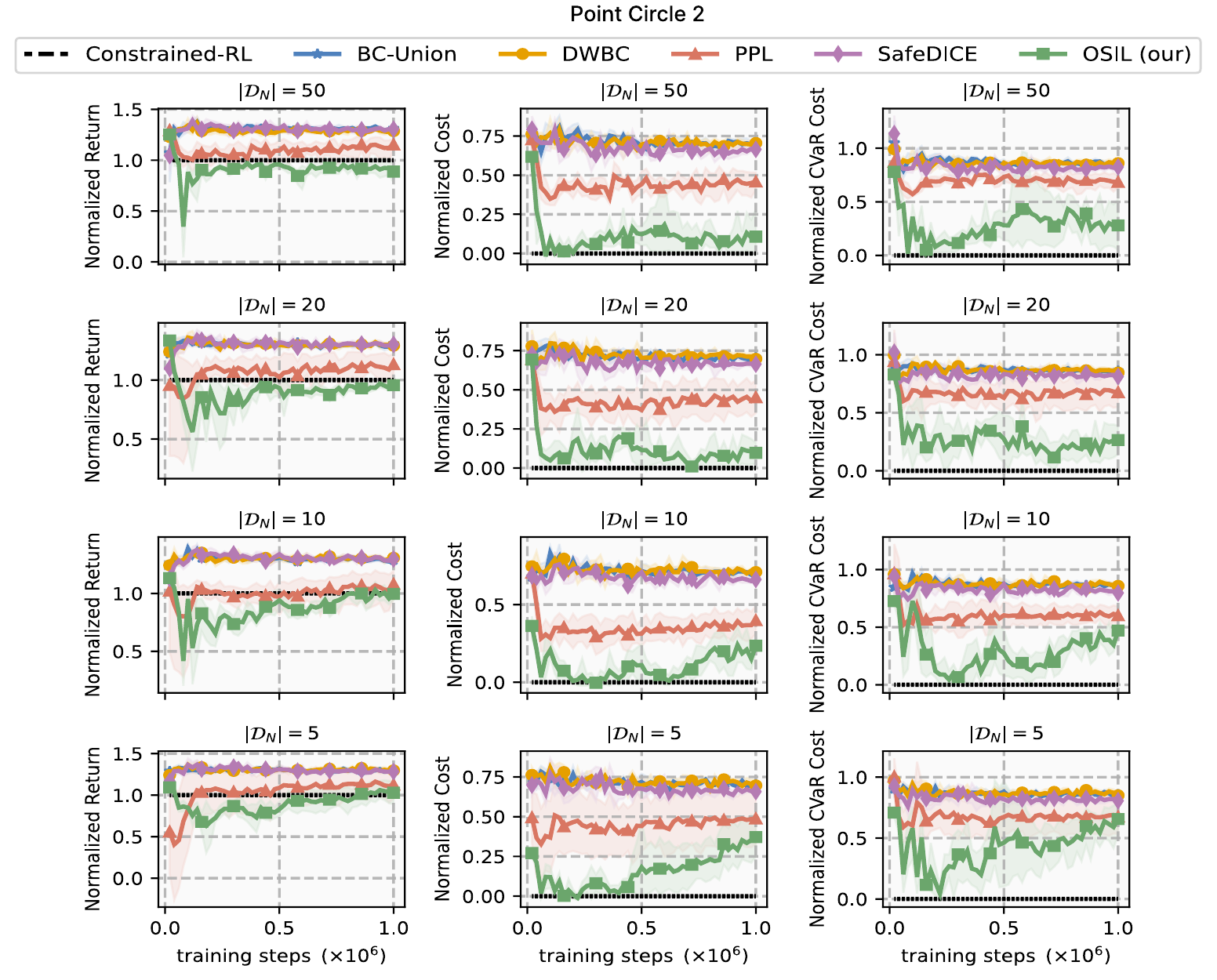}
    \caption{\textbf{Impact of Non-Preferred Trajectory Dataset Size.} Experimental result on Point-Circle2 with varying non-preferred trajectory dataset size $\mathcal{D}_N=\{5, 10, 20, 50\}$. We observe the performance gradually decreases with smaller $|\mathcal{D}_N|$. However, our approach consistently outperforms all baselines across all different dataset size.}
    \label{fig:non_preferred_dataset}
\end{figure*}

\subsection{Performance Comparison}
In both the velocity-constrained tasks (Figure \ref{fig:velocity_task_performance}) and navigation tasks (Figure \ref{fig:point_task_performance}), we observe that the standard BC policy trained on the union dataset can recover high-reward behavior. However, the resulting policy also incurs a high cost, which is expected, as the union dataset contains trajectories with varying cost levels. Our algorithm, \our, effectively recovers low-cost behavior from the union dataset in both velocity-constrained and navigation tasks, achieving safety performance comparable to that of the Constrained RL policy. In the Walker2d velocity task, most baselines succeed in recovering low-cost behavior. We find that SafeDICE performs similarly to the BC-Union policy across most domains. This can be due to its strong assumption that the union dataset should contain only low-cost and high-cost trajectories, which is a stringent requirement. On the other hand, PPL utilizes the fact that the union trajectories are generally less costly than the non-preferred ones. It learns a reward model that prefers union over non-preferred trajectories, thereby capturing some notion of safety in the state-action space. As a result, PPL outperforms SafeDICE and BC-Union but still underperforms compared to our method. 

\subsection{Ablation and Sensitivity Analysis}
\begin{figure*}[!ht]
    \centering
    \includegraphics[width=0.78\linewidth]{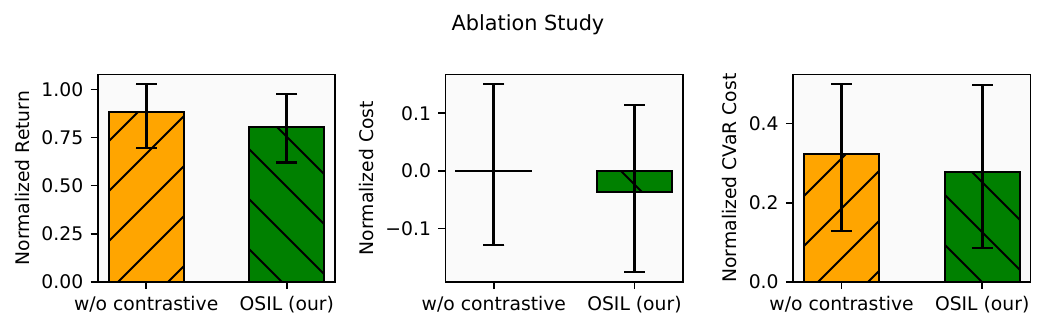}
    \caption{\textbf{Impact of Contrastive loss.} We report the mean performance of the algorithm after 1 million training steps, aggregated across both velocity-constrained and navigation tasks. Mean and 95\% CIs over 5 seeds. Our ablation highlights the significance of the contrastive loss in overall performance. See Appendix \ref{appendix:ablation} for per-task performance.}
    \label{fig:ablation_study}
\end{figure*}

\begin{figure*}[t!]
    \centering
    \includegraphics[width=0.78\linewidth]{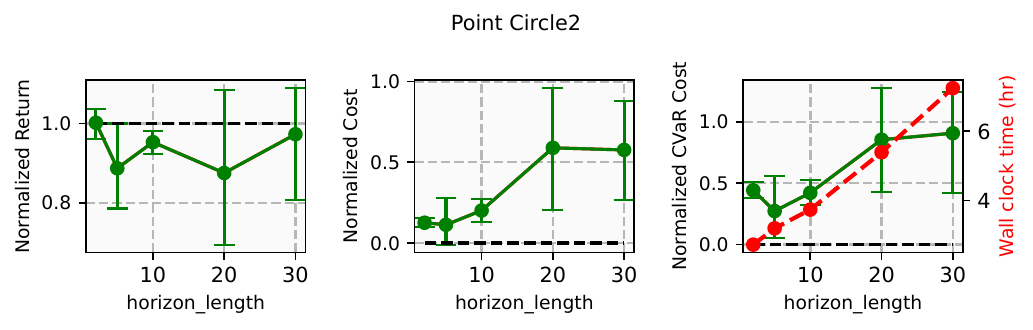}
    \caption{\textbf{Sensitivity to Trajectory Length.} We report the final mean performance of the algorithm on Point Circle2 environment for different partial trajectory length $H=\{2, 5, 10, 20, 30\}$, after training for 1 million steps. Mean and 95\% CIs over 5 seeds. We observe that the shorter trajectory length tends to better capture temporal correlations and are computationally more efficient (see rightmost plot, the red curve indicates wall-clock training time in hours).}
    \label{fig:sensitivity_hz}
\end{figure*}
\textbf{Impact of Non-Preferred Trajectory Dataset Size: } We studied the impact of non-preferred dataset size on our algorithm by varying $|\mathcal{D}_N| \in \{5, 10, 20, 50\}$ in the Point-Circle2 environment. As shown in Figure \ref{fig:non_preferred_dataset}, the performance gradually decreases with smaller $|\mathcal{D}_N|$, indicating the importance of non-preferred examples. We observe that when $|\mathcal{D}_N|$ is large, the non-preferred behaviors are well represented, we are able to recover low-cost behavior. However, as $|\mathcal{D}_N|$ decreases, the coverage of these non-preferred behaviors also decreases, making it harder for OSIL and other baselines to learn safe behavior. Nonetheless, our approach consistently outperforms all baselines across all different dataset sizes. \\

\noindent\textbf{Impact of Contrastive Loss:} To evaluate the impact of the contrastive loss term in our \our algorithm, we performed experiments on all navigational and velocity-constrained tasks. All other design choices were held constant, and we removed the contrastive loss term, $\mathcal{L}^{\text{cont}}_{\text{cost}}$, for comparison. We trained both variants of \our with and without the contrastive loss for 1 million training steps. The final performance results were aggregated across all six tasks, as shown in Figure \ref{fig:ablation_study}. The results show that the contrastive loss significantly aids in learning lower-cost policies without compromising overall return. \\

\noindent\textbf{Sensitivity to Trajectory Length:} To analyze the effect of different partial trajectory lengths $H$ on our algorithm, we performed experiments on the Point-Circle2 environment for different values of $H = \{2, 5, 10, 20, 30\}$. Figure \ref{fig:sensitivity_hz} summarizes the final performance results. We observe that shorter trajectories can effectively capture temporal correlations, which facilitates learning a better cost function and improves overall performance. Our finding also supports using partial trajectories for training, as they offer a computationally efficient alternative to full-length trajectories without compromising overall performance.

\subsection{Additional Ablations and Discussions}

We study the impact of the union dataset size (Appendix \ref{appendix:union_data_size}) and the Lagrangian penalty coefficient (Appendix \ref{appendix:lagrangian_penalty}) on our algorithm's performance. We also compare the learned cost model against the ground-truth cost and observe that our learned cost model is able to recover the true cost function for most environments (Appendix \ref{appendix:ground_truth_cost}). Furthermore, we demonstrate that \our is robust to label noise in the non-preferred dataset $\mathcal{D}_N$, i.e., when some trajectories in $\mathcal{D}_N$ do not correspond to high-cost behaviors (Appendix \ref{appendix:nosiy_non_preferred}). 
In Appendix \ref{appendix:avg_kl}, we show that minimizing the maximum KL divergence can be well approximated by minimizing the average KL divergence, Equation \ref{eq:avg_kl}, provided the union dataset $\mathcal{D}_U$ has sufficient state coverage under policy $\pi_U$. Overall \our is able to learn robust safer policy compared to other baselines.


\section{Conclusion}
\label{sec:conclusion_limitation}
This work introduces an offline safe imitation learning algorithm that learns a safe policy by inferring safety from non-preferred trajectories. We formulate the safe policy learning as a CMDP. Since we do not have access to reward and cost information, we reformulate the CMDP problem by deriving a lower bound of the reward maximizing objective and approximating the safety constraints by learning a cost model that estimates the likelihood of non-preferred behavior. We empirically demonstrate that our method can learn safer policies in the constrained RL benchmarks. A key limitation of our work is the assumption that the union dataset $\mathcal{D}_U$ contains large number of high-return trajectories. Relaxing this requirement is an important direction for future work.



\begin{acks}
Returaj Burnwal acknowledges financial support from 
the Prime Minister’s Research Fellowship (PMRF), Ministry of Education, India.
\end{acks}





\clearpage
\appendix
\onecolumn

\section{Performance Bound Proof}
\label{appendix:performance_bound_proof}

We first state all the lemmas we will need to prove Theorem \ref{theorem:bc_bound} and then provide the proof. We start by stating the Lemma 6.1 from \citet{approx_rl}, which shows that the difference in policy performance $J_r(\pi_U),\;J_r(\pi)$ can be decomposed as a sum of per-timestep advantages. We prove the lemma here for completeness. 
\begin{lemma}[\citet{approx_rl}]
    \label{lemma:policy_relation}
    For any policies $\pi_U$ and $\pi$,
    \begin{equation}
        J_r(\pi_U) = J_r(\pi) + \mathbb{E}_{\tau\sim \pi_U}\left[\sum_{t=0}^\infty \gamma^t A^\pi_r(s_t, a_t) \right]
    \end{equation}
    where $A^\pi_r(s_t, a_t) = Q^\pi_r(s_t,a_t) - V^\pi_r(s_t)$ is the advantage function for policy $\pi$.
\end{lemma}

\begin{proof}
    \begin{align}
        \mathbb{E}&_{\tau\sim \pi_U}\left[\sum_{t=0}^\infty \gamma^t A^\pi_r(s_t, a_t) \right] \\
        &= \mathbb{E}_{\tau\sim \pi_U}\left[\sum_{t=0}^\infty \gamma^t \left(Q^\pi_r(s_t,a_t) - V^\pi_r(s_t)\right) \right]\\
        &= \mathbb{E}_{\tau\sim \pi_U}\left[\sum_{t=0}^\infty \gamma^t \left(r(s_t,a_t) + \gamma V^\pi_r(s_{t+1}) - V^\pi_r(s_t)\right) \right]\\
        &= \mathbb{E}_{\tau\sim \pi_U}\left[\sum_{t=0}^\infty \gamma^t r(s_t,a_t) + \sum_{t=0}^\infty \left(\gamma^{t+1} V^\pi_r(s_{t+1}) - \gamma^t V^\pi_r(s_t)\right) \right]\\
        &\aeq \mathbb{E}_{\tau\sim \pi_U}\left[-V^\pi_r(s_0) + \sum_{t=0}^\infty \gamma^t r(s_t,a_t) \right]\\
        &= \mathbb{E}_{\tau\sim \pi_U}\left[\sum_{t=0}^\infty \gamma^t r(s_t,a_t) \right] - \mathbb{E}_{s_0}\left[-V^\pi_r(s_0)\right]\\\nonumber\\
        &= J_r(\pi_U) - J_r(\pi)
    \end{align}

    where step (a) rearranging terms in the summation via telescoping. 
\end{proof}

\noindent Based on Proposition 4.7 from \citet{levin2009markov}, we can define
\begin{definition} 
    Let $(\pi_U, \pi \mid s)$ be an $\alpha$-coupled policy pair, where $\alpha$ is the total variation (TV) divergence between the two policies at state $s$, i.e., $\alpha=D_{TV}(\pi_U(s) || \pi(s))$. Then, there exists a joint distribution over actions $(a_U, a)|s$ such that the probability of disagreement satisfy, $P(a_U \neq a|s) \leq \alpha$, where $\pi_U(s)$ and $\pi(s)$ are the marginals of $a_U$ and $a$.
\end{definition}

\noindent Similar to Lemma 2 from \citet{trpo_icml_2015}, we can bound the advantage term based on the above definition:
\begin{lemma}[\citet{trpo_icml_2015}]
    \label{lemma:adv_bound}
    Given that $\pi_U$, $\pi$ are $\alpha$-coupled policy pair, for all state $s$ and $\max_{(s,a)} |A_r^\pi(s,a)| = \epsilon$, then
    \begin{equation}
        \mathbb{E}_{a_U \sim \pi_U | s}[A^\pi_r(s,a_U)] \leq 2\alpha \epsilon
    \end{equation}
\end{lemma}
\begin{proof}
    \begin{align}
        \mathbb{E}&_{a_U \sim \pi_U | s}[A^\pi_r(s,a_U)] \\
        &\aeq \mathbb{E}_{(a_U,a) \sim (\pi_U,\pi | s)}[A^\pi_r(s,a_U) - A^\pi_r(s,a)]\\
        &= P(a_U \neq a | s) \mathbb{E}_{(a_U,a) \sim (\pi_U,\pi | s) | a_U\neq a} [A^\pi_r(s,a_U) - A^\pi_r(s,a)] \\
        &\leq 2\alpha \max_{(s,a)}|A^\pi_r(s,a)| = 2\alpha\epsilon = 2 \epsilon D_{TV}(\pi_U(s) || \pi(s))
    \end{align}
    
    where step (a) $\mathbb{E}_{a \sim \pi|s}[A^\pi_r(s,a)] = 0$
\end{proof}

\noindent Based on the work of \citet{pollard2000asymptopia}, we define the relationship between total variation divergence and KL divergence as:
\begin{lemma}[\citet{pollard2000asymptopia}]
    \label{lemma:tv_kl}
    Let $\pi_U(s)$ and $\pi(s)$ be two probability distributions defined over the same measurable space, then
    \begin{equation}
        D_{\text{TV}}(\pi_U(s)||\pi(s))^2 \leq D_{\text{KL}}(\pi_U(s)||\pi(s))
    \end{equation}
\end{lemma}
\begin{proof}
    Pinsker’s inequality states:
    \begin{equation}
        D_{\text{TV}}(\pi_U(s)||\pi(s)) \leq  \sqrt{\dfrac{1}{2} D_{\text{KL}}(\pi_U(s)||\pi(s))}
    \end{equation}
    
    Squaring both side gives:
    \begin{equation}
        D_{\text{TV}}(\pi_U(s)||\pi(s))^2 \leq \dfrac{1}{2} D_{\text{KL}}(\pi_U(s)||\pi(s)) \leq D_{\text{KL}}(\pi_U(s)||\pi(s))
    \end{equation}
\end{proof}

\noindent We now restate Theorem \ref{theorem:bc_bound} for completeness and proceed with its proof.
\begin{theorem*}[1]
    Let $\pi_U$ represents the policy that generated union dataset. Define $$\epsilon = \max_{s,a} |Q^\pi_r(s,a) - V^\pi_r(s)|,\; \text{and}\; D^{\text{max}}_{\text{KL}}(\pi_U,\pi) = \max_s D_{\text{KL}}(\pi_U(.|s) || \pi(.|s))$$
    Then the following bound holds:
    \begin{align}
        J_r(\pi) \geq J_r(\pi_U) - \dfrac{2\epsilon}{1-\gamma}\sqrt{D_{\text{KL}}^{\text{max}}(\pi_U, \pi)} 
    \end{align}
\end{theorem*}
\begin{proof}
    We begin by bounding the discounted sum of advantage term:
    \begin{align}
        \mathbb{E}_{\tau\sim \pi_U} & \left[\sum_{t=0}^\infty \gamma^t A^\pi_r(s_t, a_t) \right] \\
        \leq &\; \mathbb{E}_{\tau\sim \pi_U} \left[\sum_{t=0}^\infty \gamma^t 2\epsilon D_{\text{TV}}(\pi_U(s_t)||\pi(s_t))\right] \;\;\;\;\; \because \text{using Lemma \ref{lemma:adv_bound}} \\
        \leq &\;  \mathbb{E}_{\tau\sim \pi_U} \left[\sum_{t=0}^\infty \gamma^t\right] 2 \epsilon\max_{s} D_{\text{TV}}(\pi_U(s)||\pi(s))\\
        = &\; \dfrac{2\epsilon}{1-\gamma} \max_{s} D_{\text{TV}}(\pi_U(s)||\pi(s) \\
        \leq &\; \dfrac{2\epsilon}{1-\gamma} \max_s \sqrt{D_{\text{KL}}(\pi_U(s)||\pi(s)} \;\;\;\;\;\;\;\;\;\;\;\;\;\; \because \text{using Lemma \ref{lemma:tv_kl}} \\
        = &\;  \dfrac{2\epsilon}{1-\gamma} \sqrt{D_{\text{KL}}^{\text{max}}(\pi_U, \pi)}
    \end{align}

    Now, using Lemma \ref{lemma:policy_relation}, we can bound the policy performance as:
    \begin{align}
        J_r(\pi_U) - J_r(\pi) =&\; \mathbb{E}_{\tau\sim \pi_U}\left[\sum_{t=0}^\infty \gamma^t A^\pi_r(s_t, a_t) \right] \\
        \leq &\; \dfrac{2\epsilon}{1-\gamma} \sqrt{D_{\text{KL}}^{\text{max}}(\pi_U, \pi)}
    \end{align}
    
    Rearranging gives the desired inequality.
\end{proof}


\clearpage
\section{Environment Details}
\label{appendix:enironment_details}
We evaluated our algorithm on two categories of tasks: (i) MuJoCo-based velocity-constrained tasks: Walker-Velocity, Swimmer-Velocity, and Ant-Velocity (see Figure \ref{fig:velocity_task_env}); and (ii) navigation tasks: Point-Circle2, Point-Goal1, and Point-Button1 (see Figure \ref{fig:point_task_env}).

\begin{figure*}[h!]
    \centering
    \includegraphics[width=0.7\textwidth]{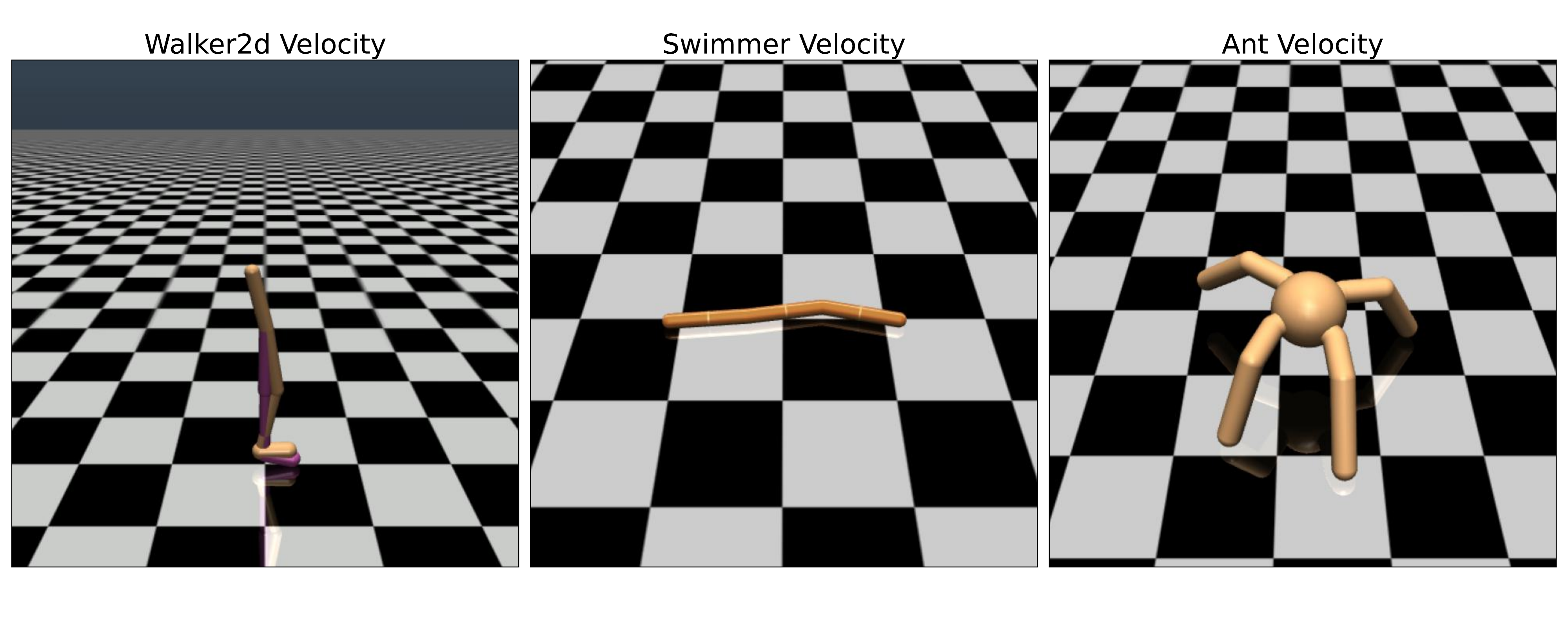}
    \caption{MuJoCo-based velocity-constrained tasks of DSRL environment. For each task, the agent needs to move as fast as possible while adhering to the velocity limits.}
    \label{fig:velocity_task_env} 
\end{figure*}
\begin{figure*}[th!]
    \centering
    \large 
    \includegraphics[width=0.7\textwidth]{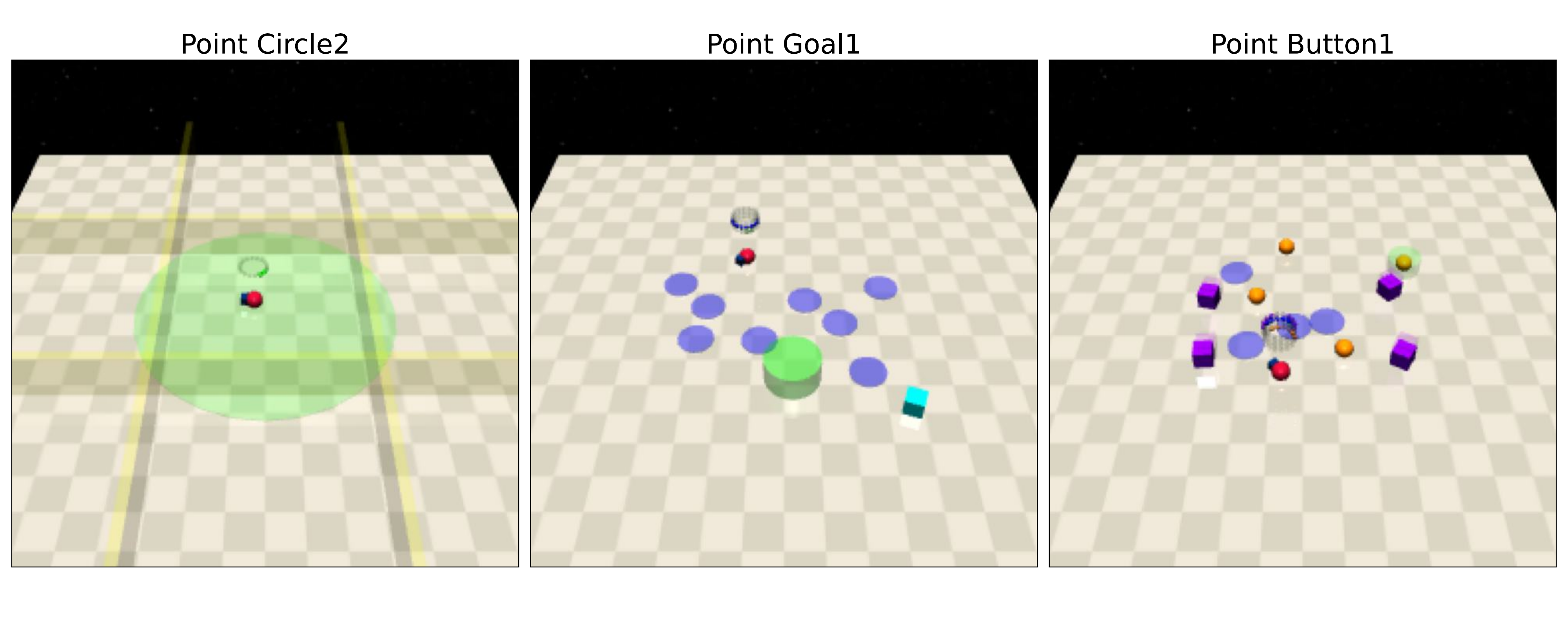}
    \caption{Navigation tasks of DSRL environment. For each task, the agent (red robot) receive costs when crossing the yellow boundaries, entering hazards indicated by blue circles or when touching the purple obstacles. The objectives of each task are as follows: \textbf{Point Circle2}: agent needs to circle around the center area as close to the boundaries for optimal reward, \textbf{Point Goal1}: Move to a series of green goal positions, \textbf{Point Button1}: Press a series of highlighted goal buttons.}
    \label{fig:point_task_env}
\end{figure*}


\clearpage
\section{Dataset Related Details}
\label{appendix:dataset_details}
To evaluate our approach, we focused on the velocity-constrained and navigation tasks from the DSRL dataset \cite{offline_safe_rl_dataset_jmlr_2024}. For union dataset we choose all the trajectories from the DSRL dataset that have total reward greater than 50\% irrespective of the trajectory's total cost. For non-preferred dataset we choose 50 randomly selected trajectories with top 30\% of total cost and total reward above 50\% 
Table \ref{table:task_specification} summarizes
the tasks, safety constraints, and dataset details for each domain used in our experiments.

\begin{table}[h!]
\begin{center}
\addtolength{\tabcolsep}{-3pt}
  \def\arraystretch{1.5}

\fontsize{9.0pt}{10.25pt}\selectfont
\captionsetup{justification=centering}
\caption{Task specification of each domain used in our experimental results}
\label{table:task_specification}

\begin{tabular}{@{}lcccccc@{}}
\toprule
                                 & \multicolumn{3}{c}{Velocity Constrained} & \multicolumn{3}{c}{Navigation}             \\ \midrule
Task Specification                        & Velocity     & Velocity    & Velocity    & Circle & Goal  & Button                    \\ 
Type of agent                             & Walker2d     & Swimmer     & Ant         & Point  & Point & Point                     \\ 
Difficulty Level                          & -            & -           & -           & 2      & 1     & 1                         \\ 
Mean cost of non-preferred trajectories   & 615.98       & 155.44      & 460.00      & 248.02 & 84.32 & 169.36                    \\ 
Mean cost of union trajectories           & 125.34       & 110.42      & 122.11      & 130.24 & 39.27 & 84.99                     \\ 
Mean return of non-preferred trajectories & 4006.76      & 202.05      & 2927.06     & 47.94  & 23.30 & \multicolumn{1}{l}{30.73} \\ 
Mean return of union trajectories         & 2810.06      & 157.82      & 2615.14     & 43.80  & 20.54 & \multicolumn{1}{l}{27.17} \\ \bottomrule
\end{tabular}
\end{center}
\end{table}


\clearpage
\section{Pseudocode of \our algorithm}
\label{appendix:pseudocode}

\begin{algorithm}[!h]
    \caption{\our}
    \label{alg:our_algorithm}
    \begin{algorithmic}[1]
        \Require Non-preferred dataset $\mathcal{D}_N$, union dataset $\mathcal{D}_U$, 
                 partial trajectory length $H$, $\eta$, $\bar{\alpha}$
        \Ensure Safe policy $\pi \approx \pi^\star$
        \State Initialize cost model $\tilde{c}$, cost action-value function $Q_{\tilde{c}}^\pi$, and policy network $\pi$
        \For{$n=1,2,\dots$}
            \State Sample bag $\mathcal{B}_N \sim \mathcal{D}_N$ and $\mathcal{B}_U \sim \mathcal{D}_U$
            \State Update cost function $\tilde{c}$ by optimizing Eq.~\ref{eq:cost_loss}
            \State Update $Q_{\tilde{c}}^\pi$ by optimizing Eq.~\ref{eq:cost_action_value_loss}
            \State Update policy $\pi$ by optimizing Eq.~\ref{eq:policy_learn_relaxed}
        \EndFor
    \end{algorithmic}
\end{algorithm}

\section{Implementation Details of the baseline algorithms}
\label{appendix:implementation_details}
In this section, we provide the implementation details of the baseline algorithms.

\subsection{BC-Union}
We train the behavior cloning \cite{bc} policy on union dataset. We minimize the following loss function:
\begin{align}
    \min_\pi -\mathbb{E}_{(s,a)\sim \mathcal{D}_U} [\log \pi(a|s)]
\end{align}

\subsection{DWBC}
We modified DWBC \cite{dwbc_icml_2022} and used negative-unlabeled learning to train the discriminator model as follows:
\begin{align}
    \min_d\; &\eta \mathbb{E}_{(s,a)\sim\rho^N} [-\log d(s,a, \log \pi)] \nonumber\\
    &+ \mathbb{E}_{(s,a)\sim\rho^U}[-\log(1-d(s,a,\log \pi))] \nonumber\\
    &- \eta\mathbb{E}_{(s,a)\sim\rho^N}[-\log(1-d(s,a,\log\pi))]
\end{align}
where $d$ is the discriminator model and $\eta$ is the hyperparameter. We then train safe policy by using this discriminator model to construct weight in the weighted BC loss function as follows:
\begin{align}
    \min_\pi \mathbb{E}_{(s,a)\sim\rho^U}[(1-d(s,a))\mathcal{L}_\pi(s,a)]
\end{align}

\subsection{PPL}
We train the reward function based on  Bradley-Terry model \cite{bradley_terry_1952}, which is as follows:
\begin{align}
    \min_r - \sum_{(i,j)\in \mathcal{P}} \log \dfrac{\exp\left( \sum_{(s,a)\in\tau_i} r(s,a) \right)}{\exp\left( \sum_{(s,a)\in\tau_i} r(s,a) \right) + \exp\left( \sum_{(s,a)\in\tau_j} r(s,a) \right)}
\end{align}
where $\tau_i$ and $\tau_j$ are trajectories and $\mathcal{P} = \{(i,j): \tau_i \succ \tau_j\}$. In our setting we assumed that union demonstrations are better than non-preferred demonstrations, i.e. $\tau_i \succ \tau_j \;\forall\; \tau_i \sim \mathcal{D}_U, \tau_j \sim \mathcal{D}_N$. Based on this reward function we train our safe policy as follows:
\begin{align}
    \min_\pi -\mathbb{E}_{(s,a)\sim\mathcal{D}_U}[r(s,a)\log\pi(a|s)]
\end{align}

\subsection{SafeDICE}
SafeDICE \cite{safedice_neurips_2023} algorithm first estimates the log ratio of $\rho^N(s,a)$ and $\rho^U(s,a)$ by training a discriminator model:
\begin{align}
    c^* = \text{arg}\max_c \mathbb{E}_{(s,a)\sim\mathcal{D}_N}[\log c(s,a)] + \mathbb{E}_{(s,a)\sim\mathcal{D}_U}[\log (1-c(s,a))]
\end{align}
The discriminator model $c^*$ is then used to compute the log ratio:
\begin{align}
    \text{log ratio} = r_\alpha(s,a) = \log \dfrac{1-(1+\alpha)c^*(s,a)}{(1-\alpha)(1-c^*(s,a))}
\end{align}
The log ratio estimate is then used to estimate $\nu$ network:
\begin{align}
    \min_\nu (1-\gamma)\mathbb{E}_{s\sim\rho_0}[\nu(s)] + \log\mathbb{E}_{(s,a,s')\sim\mathcal{D}_U}\left[\exp(A_\nu(s,a,s'))\right]
\end{align}
where $A_\nu(s,a,s') = r(s,a) + \gamma \nu(s')$. Then, the safe policy is estimated as follows:
\begin{align}
    \min_\pi - \dfrac{\mathbb{E}_{(s,a,s') \sim \mathcal{D}_U}[w_\nu(s,a,s')\mathcal{L}_\pi(s,a)]}{\mathbb{E}_{(s,a,s')\sim\mathcal{D}_U}[w_\nu(s,a,s')]}
\end{align}
where $w_\nu(s,a,s') = \exp(A_{\nu^*}(s,a,s'))$ and $\mathcal{L}_\pi$ denotes the BC loss function.

\noindent We use the author's implementation of SafeDICE as a baseline in our work: \url{https://github.com/jys5609/SafeDICE}

\subsection{Constrained Decision Transformer}
An offline constrained-RL algorithm, Constrained Decision Transformer (CDT) \cite{cdt_icml_2023}, requires access to both cost and reward annotation to train. We use CDT to compare our learned policies performance. We use the CDT implementation from OSRL \cite{offline_safe_rl_dataset_jmlr_2024} library. Github link: \url{https://github.com/liuzuxin/OSRL/tree/main}
 
\section{Hyperparameter configurations}
\label{appendix:hyperparameter}
For fair comparison, we use the same architecture and learning rate to train the policy, critic, reward, and cost models of each algorithm. Table \ref{table:hyperparameter} summarizes the hyperparameter configurations that we used in our experiments.
We use the same hyperparameters throughout our experiments, except where explicitly stated.

\begin{table}[!htb]
\centering
\caption{Hyperparameters used in our experimental results}
\label{table:hyperparameter}

\begin{tabular}{@{}cccccc@{}}
\toprule
\multirow{2}{*}{Hyperparameter}       & \multicolumn{5}{c}{Algorithm}                                                                                                                                                              \\ \cmidrule(l){2-6} 
                                      & BC-Union           & DWBC               & PPL                & SafeDICE           & \our (ours)                                                                             \\ \midrule\midrule
$\gamma$ (discount factor)            & 0.99               & 0.99               & 0.99               & 0.99               & 0.99                                                                                                   \\
learning rate (actor)                 & $1 \times 10^{-5}$ & $1 \times 10^{-5}$ & $1 \times 10^{-5}$ & $1 \times 10^{-5}$ & $1 \times 10^{-5}$                                                                                     \\
network size (actor)                  & $[256, 256]$       & $[256, 256]$       & $[256, 256]$       & $[256, 256]$       & $[256, 256]$                                                                                           \\
learning rate (cost)                  & -                  & -                  & -                  & $1 \times 10^{-5}$ & $1 \times 10^{-5}$                                                                                     \\
network size (cost)                   & -                  & -                  & -                  & $[256, 256]$       & $[256, 256, 128]$                                                                                      \\
learning rate (reward / critic model) & -                  & $1 \times 10^{-5}$ & $1 \times 10^{-5}$ & $1 \times 10^{-5}$ & $1 \times 10^{-5}$                                                                                     \\
network size (reward / critic model)  & -                  & $[256, 256]$       & $[256, 256]$       & $[256, 256]$       & $[256, 256]$                                                                                           \\
gradient penalty coefficient          & -                  & -                  & -                  & 10.0               & 1.0                                                                                                    \\
weight decay                          & 0.01               & 0.01               & 0.01               & 0.01               & 0.01                                                                                                   \\
$\eta$                                & -                  & 0.5                & -                  & -                  & 0.1                                                                                                    \\
$\bar{\alpha}$                        & -                  & -                  & -                  & -                  & \begin{tabular}[c]{@{}c@{}}0.01 (Velocity Tasks), \\ 0.005 (Navigation Tasks)\end{tabular} \\
trajectory length (H)                 & -                  & -                  & 5                  & -                  & 5                                                                                                     \\ 
batch size                            & 128                & 128                & 128                & 128                & 128                                                                                                    \\
\# training steps                     & $1,000,000$        & $1,000,000$        & $1,000,000$        & $1,000,000$        & $1,000,000$                                                                                            \\
optimizer                             & adam               & adam               & adam               & adam               & adam                                                                                                   \\ \bottomrule
\end{tabular}
\end{table}


\clearpage
\section{Additional Experiments}
\label{appendix:additional_experiments}
\subsection{Impact of contrastive loss}
\label{appendix:ablation}
In this section we report the results for all the environments in Figure \ref{fig:ablation_velocity}, \ref{fig:ablation_navigation}.
We observe that contrastive loss significantly aids in learning lower-cost policies except for Swimmer-Velocity task without compromising overall return.  

\begin{figure*}[h!]
    \centering
    \includegraphics[width=0.75\textwidth]{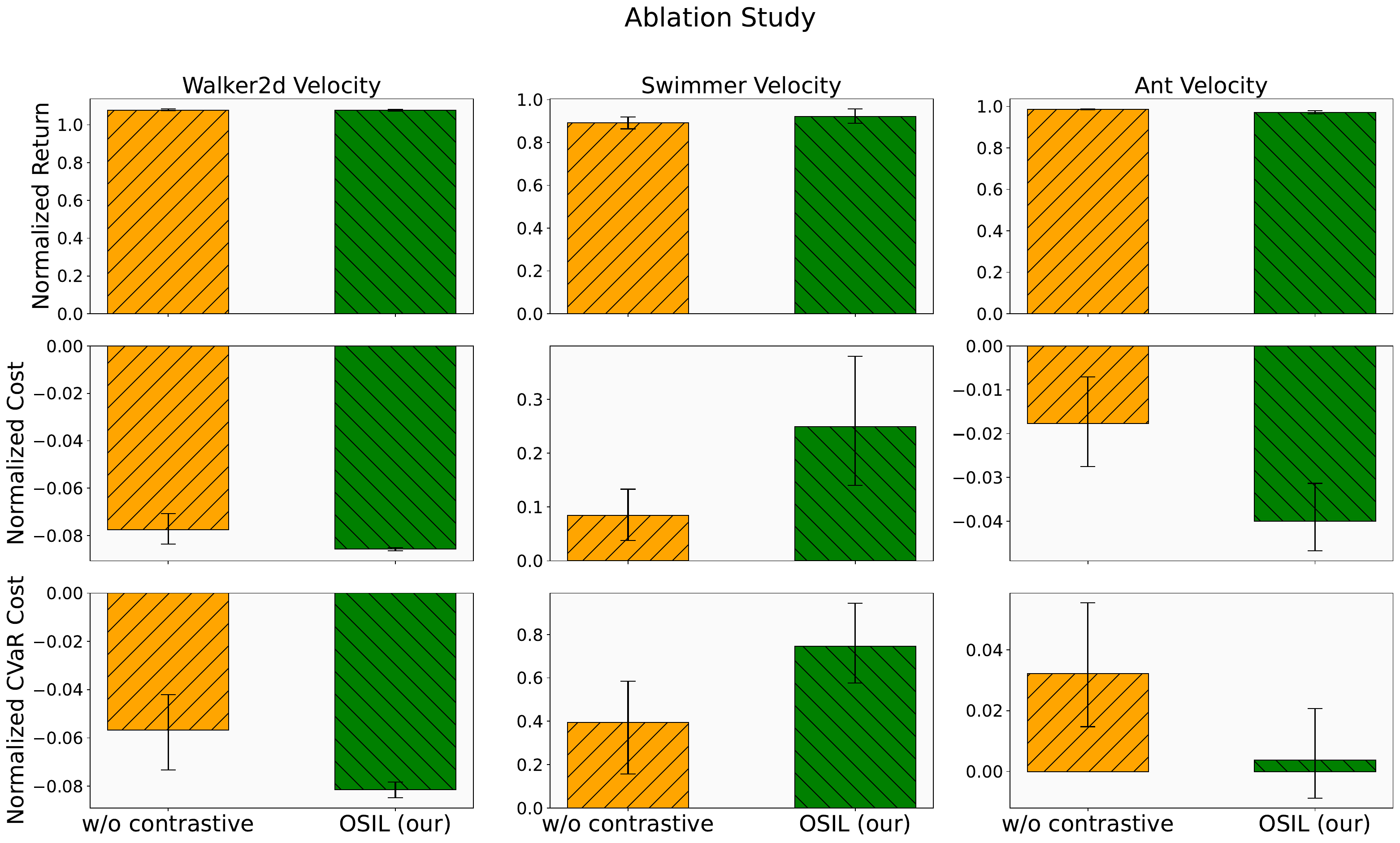}
    \caption{\textbf{Impact of Contrastive loss.} We report the mean performance of the algorithm after 1 million training steps for velocity constrained tasks. Mean and 95\% CIs over 5 seeds.}
    \label{fig:ablation_velocity}
\end{figure*}

\begin{figure*}[h!]
    \centering
    \includegraphics[width=0.75\textwidth]{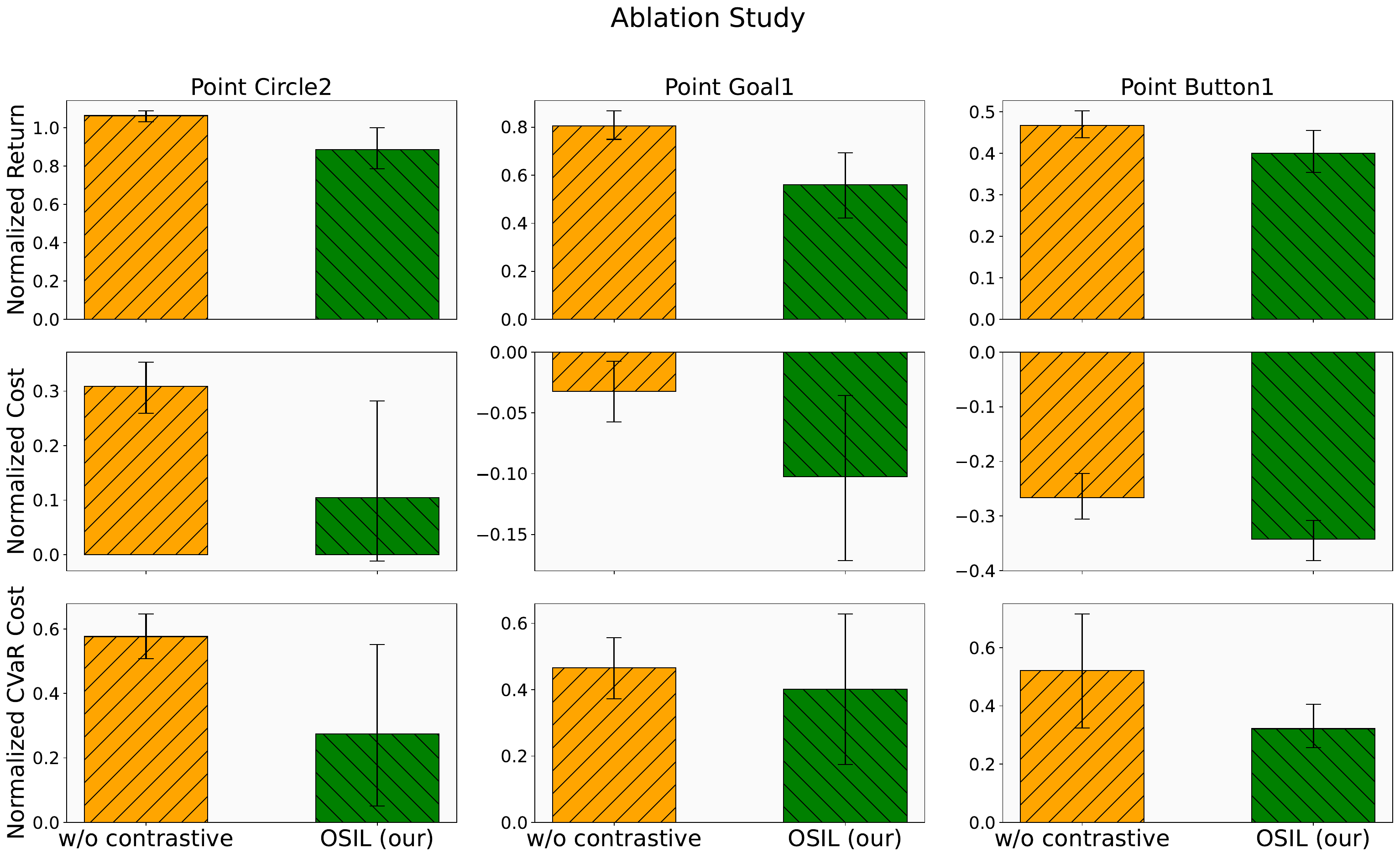}
    \caption{\textbf{Impact of Contrastive loss.} We report the mean performance of the algorithm after 1 million training steps for navigation tasks. Mean and 95\% CIs over 5 seeds.}
    \label{fig:ablation_navigation}
\end{figure*}

\subsection{Impact of Union dataset size}
\label{appendix:union_data_size}
We conducted experiment on Point-Circle2 environment to study the impact of $\mathcal{D}_U$ dataset size on the safety performance. We observe the performance gradually decreases with smaller $|\mathcal{D}_U|$. 
\begin{figure*}[th!]
    \centering
    \includegraphics[width=0.63\textwidth]{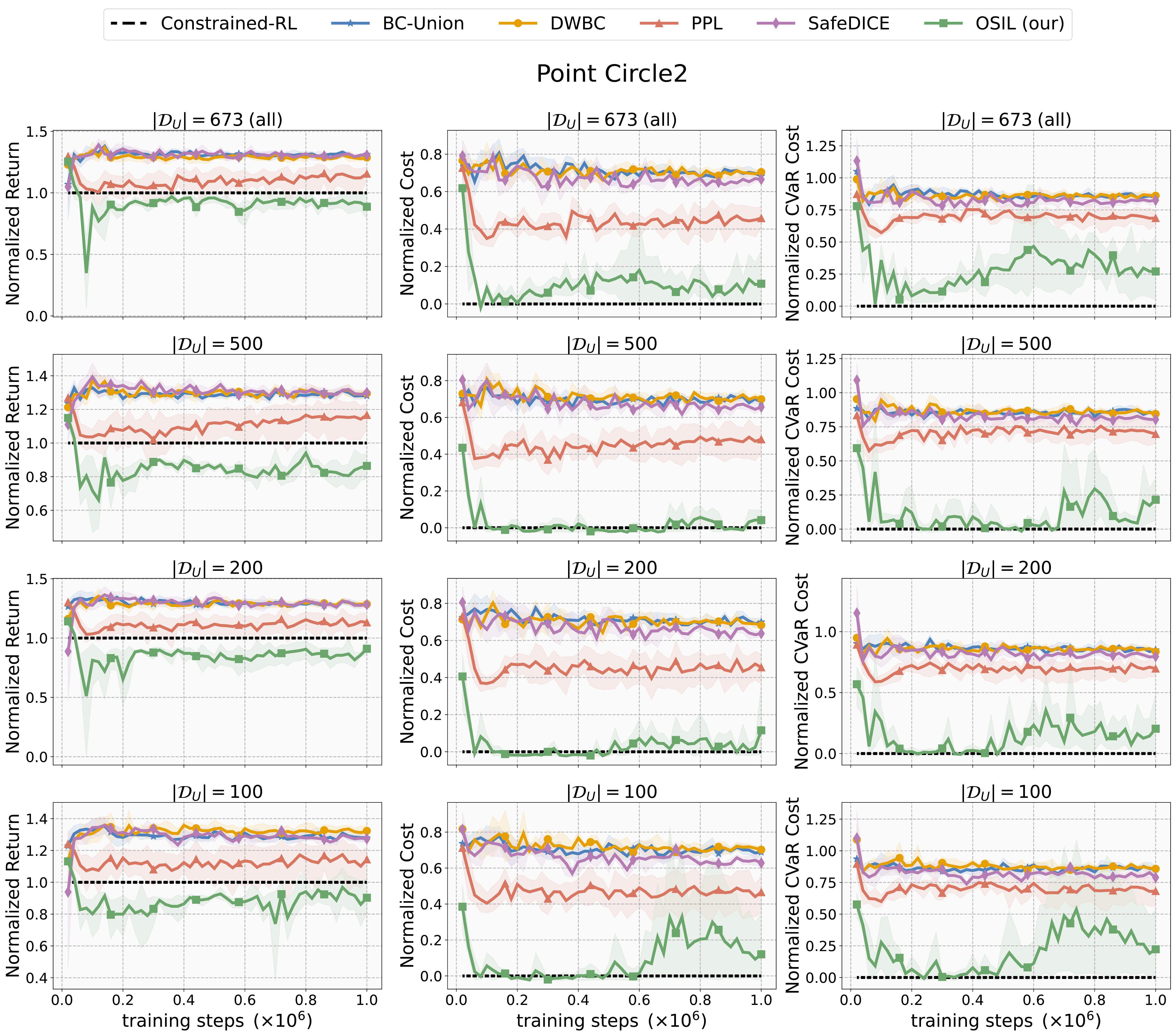}
    \caption{\textbf{Impact of Union Trajectory Dataset Size.} Experimental result on Point-Circle2 with varying union trajectory dataset size $|\mathcal{D}_U|=\{100, 200, 500, 673\; (\text{all})\}$ and $|\mathcal{D}_N| = 50$. Constrained-RL algorithm is trained on all 673 union trajectories. We observe the performance gradually decreases with smaller $|\mathcal{D}_U|$. However, our approach consistently outperforms all baselines across all different dataset size.}
\end{figure*}


\subsection{Sensitivity to Lagrangian Penalty Coefficient}
\label{appendix:lagrangian_penalty}
We conducted experiments under two different settings: (1) treating $\alpha$ in Equation \ref{eq:policy_learn_relaxed} as a fixed hyperparameter, and (2) using an adaptive version of $\alpha$ as defined in Equation \ref{eq:adaptive_alpha}, where $\bar{\alpha}$ serves as the hyperparameter. As shown in Figure \ref{fig:lagrangian_penalty}, while the adaptive approach using $\bar{\alpha}$ is sensitive to value changes, it can consistently learn safer policies with lower cost than using a fixed $\alpha$.
\begin{figure*}[h!]
    \centering
    \includegraphics[width=0.67\textwidth]{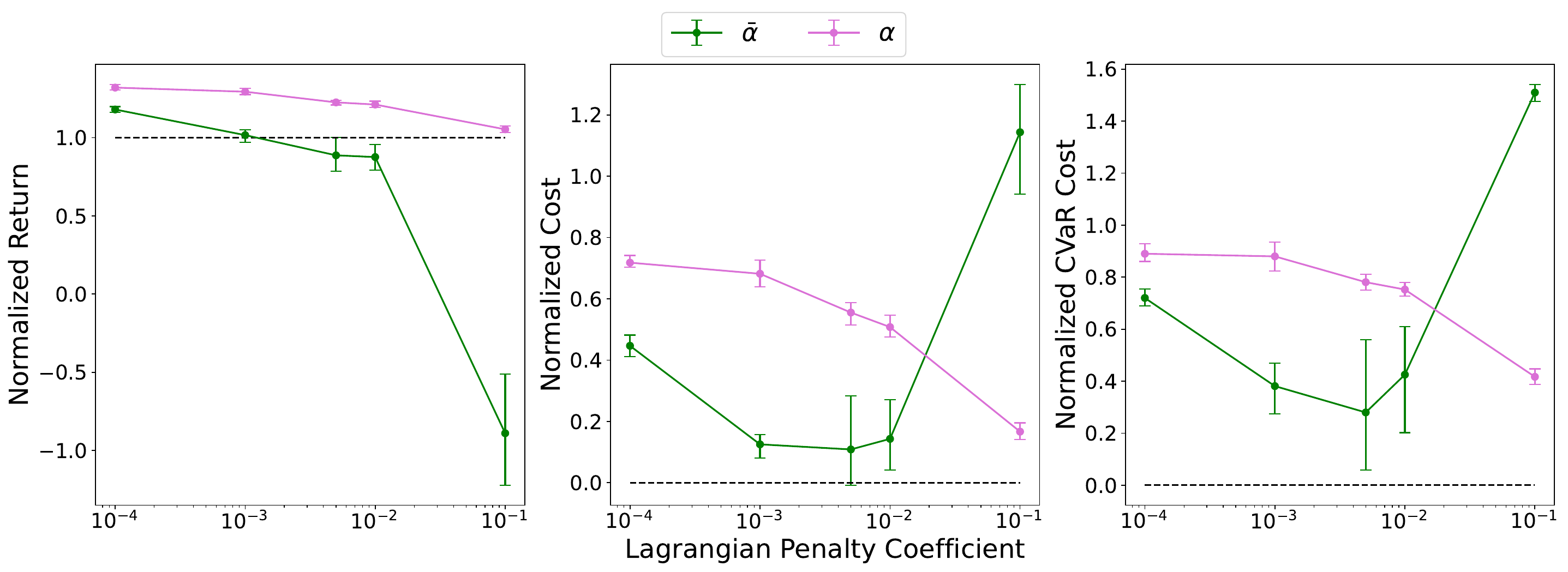}
    \caption{\textbf{Sensitivity to Lagrangian Penalty Coefficient.} Experimental result on Point-Circle2 with different penalty coefficient $\{0.0001, 0.001, 0.005, 0.01, 0.1\}$, after training for 1 million steps. Mean and 95\% CIs over 5 seeds. }
    \label{fig:lagrangian_penalty}
\end{figure*}


\subsection{Comparison with ground truth cost}
\label{appendix:ground_truth_cost}
In Figures \ref{fig:w_wo_ground_cost_velocity}, \ref{fig:w_wo_ground_cost_navigation}, we compare \our algorithm with and without access to ground-truth cost. When the ground-truth cost is available, we bypass cost model learning and directly use the true cost information within the \our algorithm. For the velocity-constrained task, the performance of \our with and without ground-truth cost is nearly identical. In the Point-Circle2 task, \our with ground-truth cost achieves better performance. However, in the Point-Goal1 and Point-Button1 environments, \our without ground-truth cost learns low-cost policies, though at the expense of lower reward, compared to \our with ground-truth cost.

\begin{figure*}[h!]
    \centering
    \includegraphics[width=0.68\textwidth]{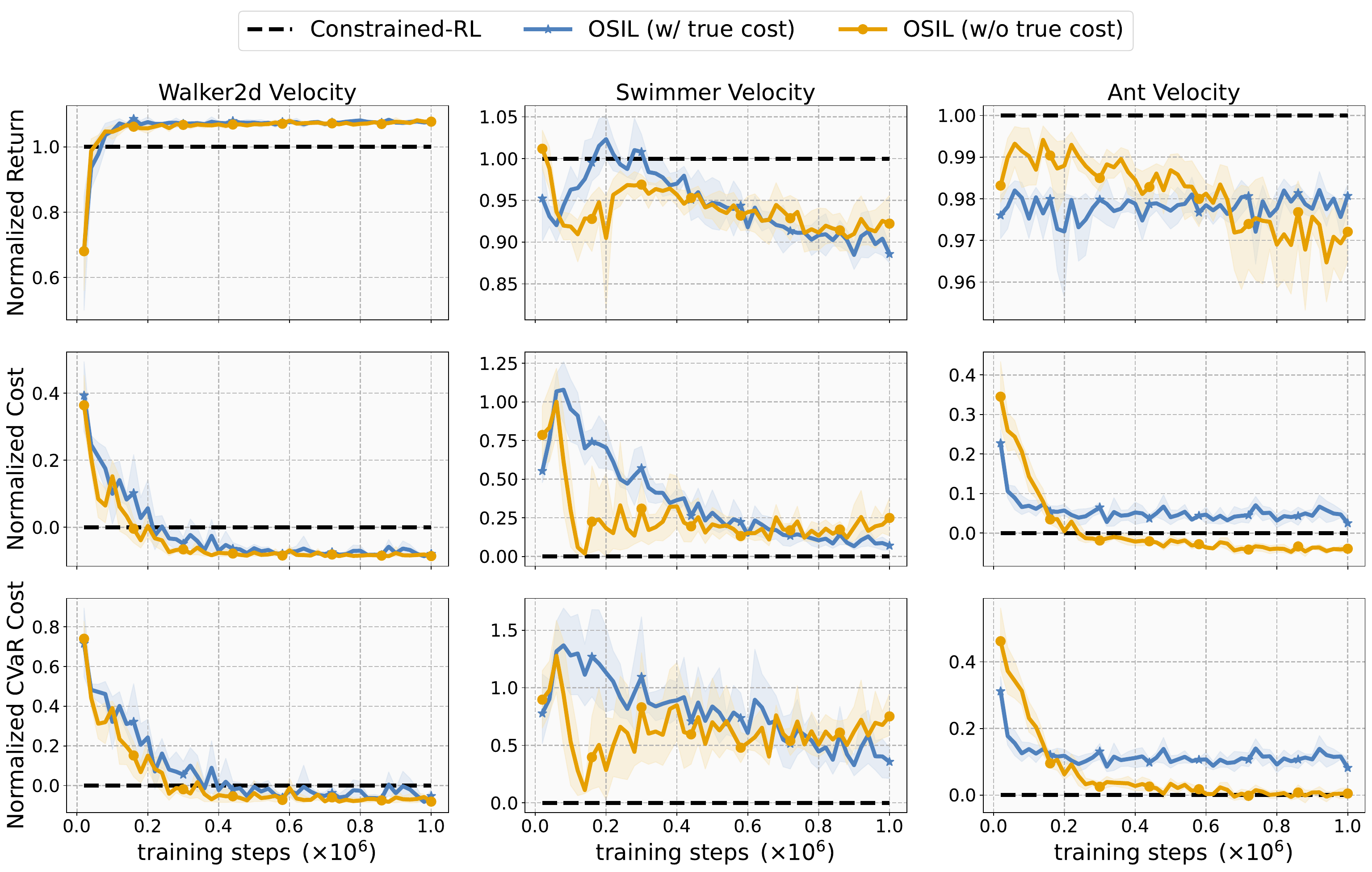}
    \caption{\textbf{Comparison of \our with and without ground truth cost.} Experimental result on Walker2d-Velocity, Swimmer-Velocity, Ant-Velocity. The shaded area represents the standard error.}
    \label{fig:w_wo_ground_cost_velocity}
\end{figure*}

\begin{figure*}[h!]
    \centering
    \includegraphics[width=0.68\textwidth]{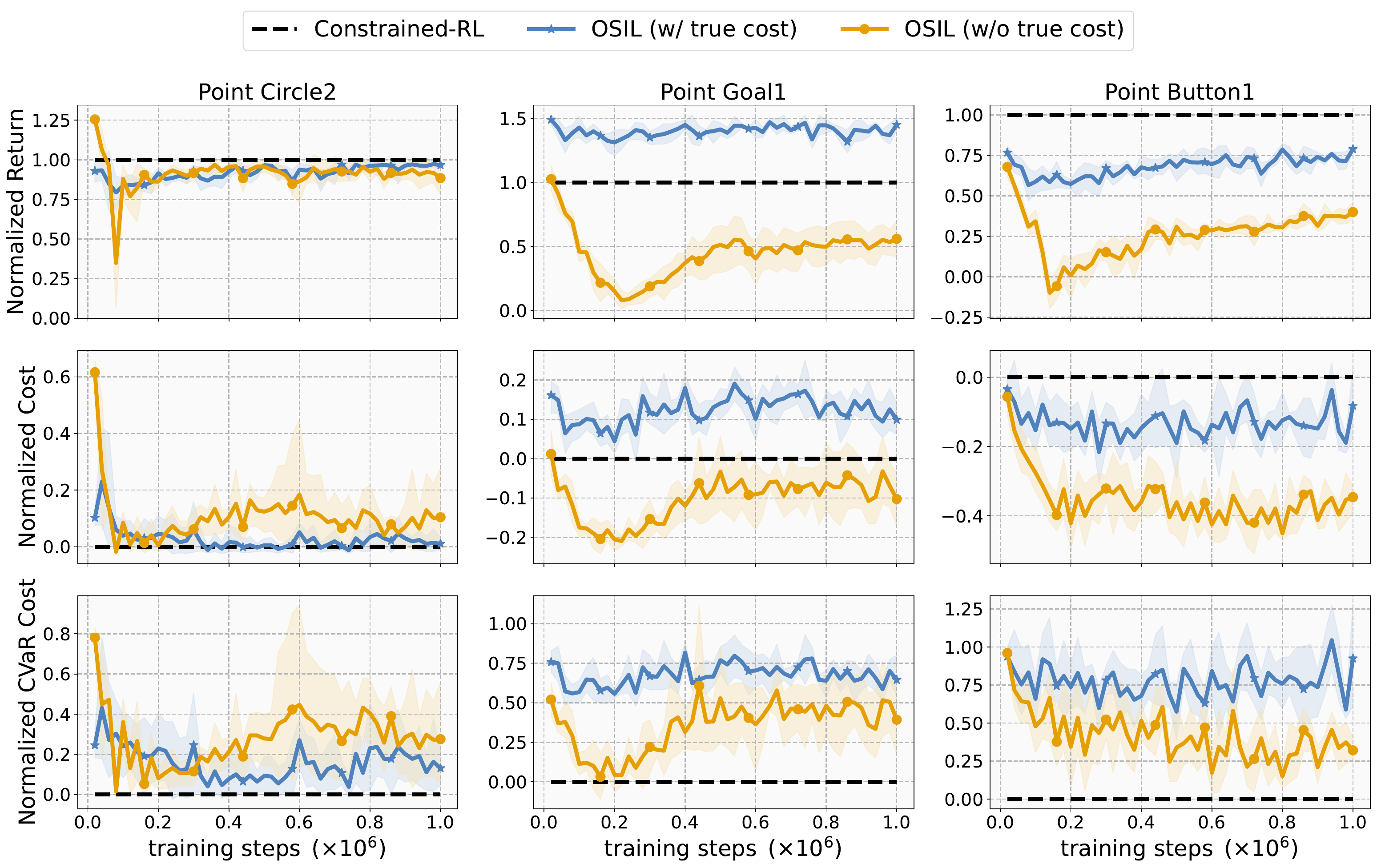}
    \caption{\textbf{Comparison of \our with and without ground truth cost.} Experimental result on Point-Circle2, Point-Goal1, Point-Button1. The shaded area represents the standard error. }
    \label{fig:w_wo_ground_cost_navigation}
\end{figure*}

In Figure \ref{fig:cost_pred_velocity}, \ref{fig:cost_pred_navigation} we compare the true total cost of the trajectory with the predicted total cost of the trajectory from the union dataset. We estimate the predicted total cost of the trajectory using the learned cost model $\tilde{c}$. We observe that for Walker2d-Velocity, Ant-Velocity, Point-Circle2 tasks our learned cost model is able to recover the true total cost of the trajectory. We are partially able to recover the true total cost of the trajectory for Swimmer-Velocity, Point-Goal1 and Point-Button tasks. 
\begin{figure*}[h!]
    \centering
    \includegraphics[width=0.7\textwidth]{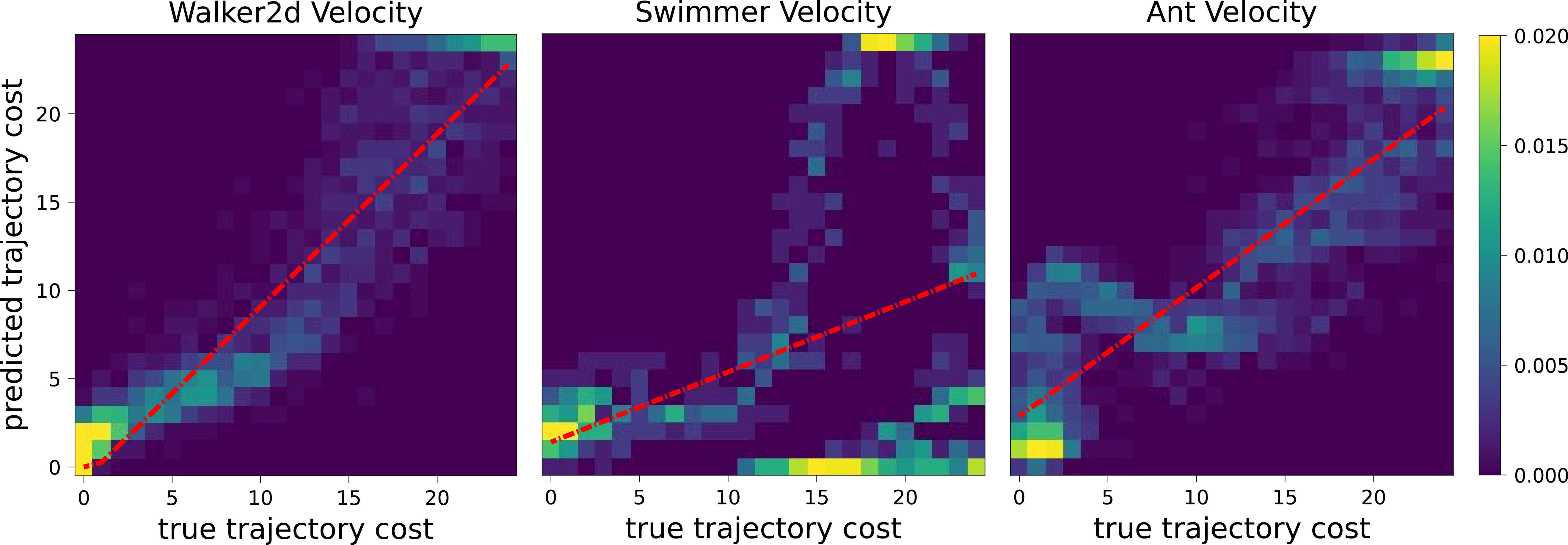}
    \caption{\textbf{Comparison of ground truth cost with predicted cost of the union trajectory.} Experimental result on Walker2d-Velocity, Swimmer-Velocity, Ant-Velocity.}
    \label{fig:cost_pred_velocity}
\end{figure*}

\begin{figure*}[h!]
    \centering
    \includegraphics[width=0.7\textwidth]{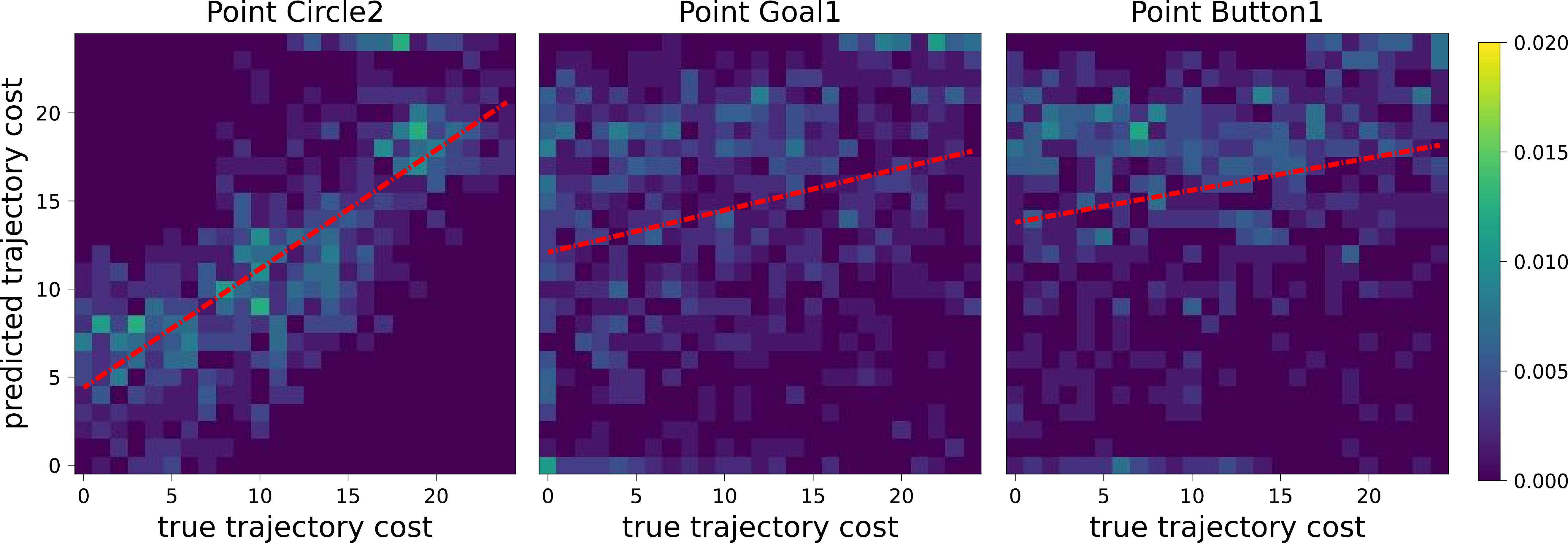}
    \caption{\textbf{Comparison of ground truth cost with predicted cost of the union trajectory.} Experimental result on Point-Circle2, Point-Goal1, Point-Button1.}
    \label{fig:cost_pred_navigation}
\end{figure*}


\subsection{Performance on Noisy Non-Preferred Dataset.}
\label{appendix:nosiy_non_preferred}
\begin{figure*}[h!]
    \centering
    \includegraphics[width=0.7\textwidth]{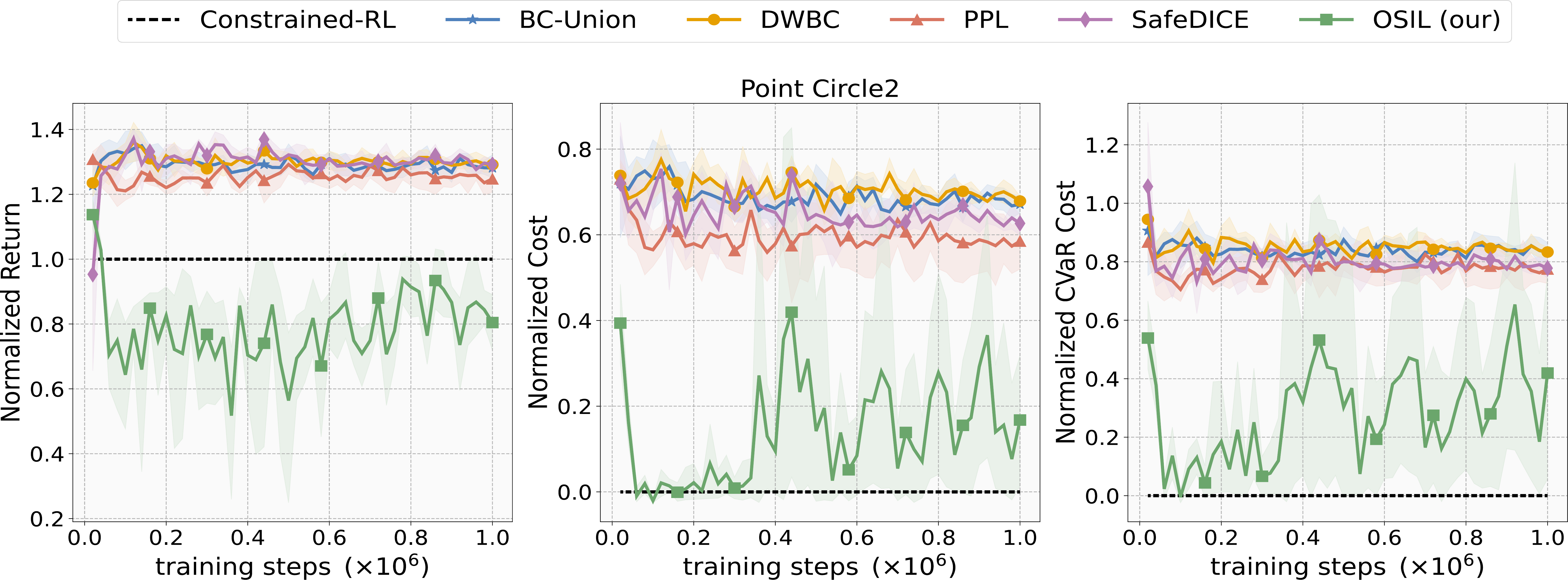}
    \caption{\textbf{Performance on Noisy Non-Preferred Dataset.} We conducted experiment on Point-Circle2 environment, where 20\% of the trajectories in the non-preferred dataset were mislabeled, i.e., these 20\% trajectories are not high-cost but samples from union dataset. We observe that our proposed algorithm is able to outperform all the baselines.}
    \label{fig:noisy_20p_non_preferred}
\end{figure*}

\begin{figure*}[th!]
    \centering
    \includegraphics[width=0.7\textwidth]{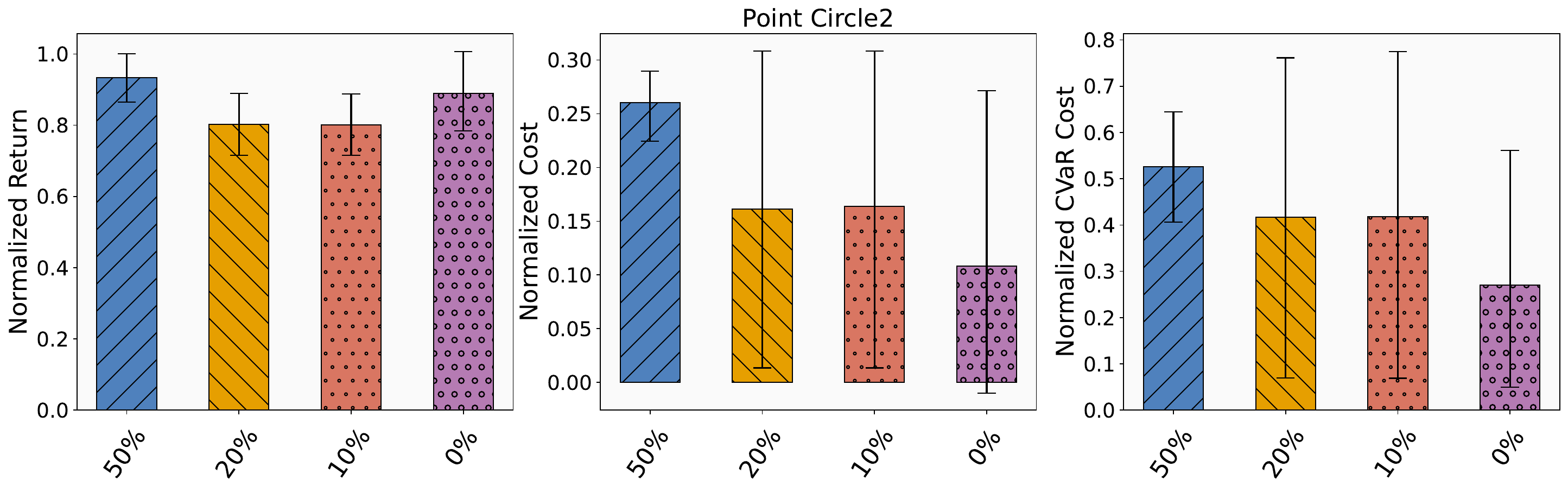}
    \caption{\textbf{Performance on Noisy Non-Preferred Dataset.} We conducted experiment with different percentages of noise (i.e, noise = $\{0\%, 10\%, 20\%, 50\%\}$) in the $\mathcal{D}_N$ dataset for Point-Circle2 and report the mean performance of the \our after training for 1 million steps. We observe that as the noise level increases, \our performance also decreases. Mean and 95\% CIs over 5 seeds.}
    \label{fig:noisy_non_preferred}
\end{figure*}


\section{Additional Discussions}
\label{appendix:additional_discussions}
\subsection{Practical Implication of Minimizing Average KL divergence}
\label{appendix:avg_kl}

We approximate the minimization of maximum KL with the average KL divergence, Equation \ref{eq:avg_kl}. The union dataset $\mathcal{D}_U$ contains states that are generated by executing policy $\pi_U$. We want to minimize $\max_s D_{KL}(\pi_U(s)\; ||\; \pi(s))$. However, computing the maximum is often impractical due to the many constraints. To address this, we approximate it by minimizing the average KL divergence, defined in equation \ref{eq:avg_kl} as $\mathbb{E}_{s \sim \mathcal{D}_U} [D_{KL}(\pi_U(s)\; ||\; \pi(s))]$. The average KL divergence is defined as an expectation over $\mathcal{D}_U$; it is not computed over a uniform distribution across the entire state space. While minimizing the average KL divergence is more tractable, it does not guarantee that the KL divergence is small for all states. Some states may still exhibit high KL divergence. However, such states will typically have low visitation probability under $\pi_U$. If a state is frequently visited under $\pi_U$, minimizing average KL would force the KL divergence at that state to be small. Thus, even though some rarely visited states may have large divergences, this is usually acceptable in practice because the agent is unlikely to encounter them.

We estimate the average KL divergence using the union dataset $\mathcal{D}_U$. If $\mathcal{D}_U$ lacks sufficient coverage of the state encountered by policy $\pi_U$, then in such cases, the KL divergence can be high for states that are visited under the policy $\pi_U$. This can lead the learned policy to hallucinate actions for underrepresented states, causing its behavior to diverge from that of $\pi_U$. 

\subsection{Comparison of \texorpdfstring{$\alpha$}{alpha} in TD3+BC with \our}
\label{appendix:compare_alpha}

\our formulation looks similar to TD3+BC. The TD3+BC \cite{td3_plus_bc_neurips_2021} objective is defined as:
\begin{align}
    \min_\pi &- \mathbb{E}_{(s,a) \sim \mathcal{D}_U} \left[ (\pi(s) - a)^2 + \alpha Q^\pi_r(s, \pi(s))\right] \nonumber \\
    & \text{where, } \alpha = \dfrac{\bar{\alpha}}{\dfrac{1}{N} \sum_{s_i,a_i\sim \mathcal{D}_U} |Q^\pi_r(s_i, a_i)|} 
\end{align}
In TD3+BC, the range of values that $Q^\pi_r$ can take depends on the reward, whereas the BC term can take value of at most 4 if the action range is assumed to be $[-1, 1]$. The $\alpha$ in TD3+BC is normalization term based on the average absolute value of $Q^\pi_r$. It balances the scale between the values of $Q^\pi_r$ and the BC term. However, in \our, $\alpha$ balances the trade-off between safety (minimizing cost-action value) and performance (maximizing reward through imitation). 


\subsection{Inclusion of low-reward trajectories in Non-Preferred Dataset}
\label{appendix:non_preferred_low_reward}

Offline safe IL algorithms finds it difficult to learn safe policy when low-reward trajectories are included in non-preferred dataset $\mathcal{D}_N$ because of following reasons:
\begin{enumerate}
    \item Low-reward trajectories can arise from a wide range of suboptimal behaviors depending on the environment. Consider an example where the agent receives low reward can result from numerous behaviors such as standing still, moving back and forth, circling, or heading in any direction other than toward the goal. To effectively learn to avoid such multi-mode behaviors, the non-preferred dataset must capture this diversity through sufficient coverage. However, with only a limited number of non-preferred trajectories, the policy may find it difficult to avoid these non-preferred behaviors. 
    \item The non-preferred dataset is a subset of the union dataset. When low-reward trajectories are included in the non-preferred set, the union dataset then contains a mix of both high-reward and low-reward trajectories with varying costs, essentially covering the entire space of possible trajectories. However, since explicit reward and cost labels are unavailable, and the non-preferred dataset is limited, it becomes difficult to identify which behavior the agent should imitate and which to avoid. As a result, learning a safe policy in this setting is challenging.
\end{enumerate}

We conducted experiments on Point-Circle2 and Point-Goal1 environments, Figure \ref{fig:highcost_lowreward}. In this setup, the non-preferred dataset consisted of high-cost and low-reward trajectories. The union dataset included a mix of high and low reward trajectories with varying costs, i.e., effectively covering the full spectrum of possible trajectories. Across both environments, we observe that none of the algorithms could learn safer policies. Although the union dataset includes all available trajectories, it is dominated by high-return trajectories (See DSRL dataset \cite{offline_safe_rl_dataset_jmlr_2024}). As a result, all algorithms tend to exhibit high-return behavior. However, since high-cost behaviors are underrepresented in the limited non-preferred dataset, the algorithms fail to learn safe behavior. While inclusion of low-reward trajectories in non-preferred dataset remain a challenging case for all offline safe IL methods, this lies outside our target setting, where non-preferred trajectories primarily capture high-cost behaviors. We leave the inclusion of low-reward trajectories in the non-preferred dataset to future work.

\begin{figure*}[th!]
    \centering
    \includegraphics[width=0.77\textwidth]{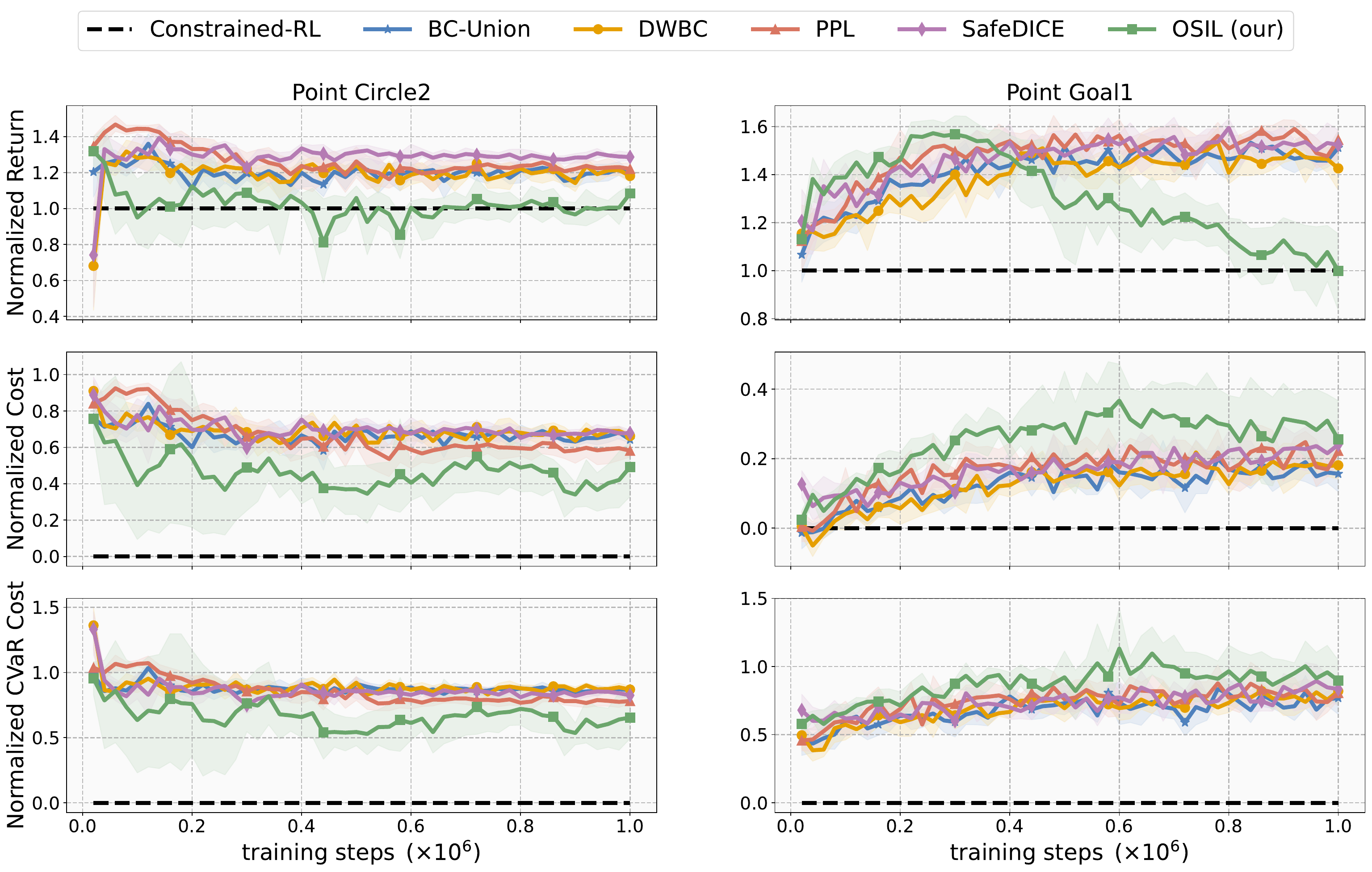}
    \caption{\textbf{Inclusion of low-reward trajectories in Non-Preferred Dataset.} For both Point-Circle2 and Point-Goal1 environments, high-cost behaviors are underrepresented in the limited non-preferred dataset $\mathcal{D}_N$, we observe that none of the algorithms could learn safe policies. Mean and 95\% CIs over 5 seeds.}
    \label{fig:highcost_lowreward}
\end{figure*}


\section{Compute Details}
\label{appendix:compute}

We used two server nodes equipped with the following specification:
\begin{itemize}
    \item CPU: AMD EPYC 7543 32-Core Processor
    \item Memory: 30GB
    \item GPU: NVIDIA A30
\end{itemize}

\section{Broader Impact}
\label{appendix:broader_impact}
This work lays an algorithmic foundation for safe imitation learning in offline settings. In many real-world domains, specifying accurate safety costs can be difficult. However, it is often feasible to collect trajectories that reflect undesirable or unsafe behavior, implicitly conveying what the agent should avoid. Our approach learns safe behavior from these trajectories. To effectively learn safe behavior, our approach requires sufficient coverage of the space of undesirable behaviors; if key undesirable behaviors are missing from the dataset, the agent may fail to learn to avoid them. Our work has strong potential for real-world applications where specifying safety criteria is difficult, such as autonomous driving or toxicity moderation in large language models. However, because our method depends on human-provided examples of undesired behavior, it is susceptible to encoding human biases. Therefore, regulatory oversight might be necessary to mitigate these biases during non-preferred dataset collection. 

\end{document}